\documentclass{article}
\usepackage[margin=1in]{geometry}
\usepackage{authblk}

\usepackage{hyperref}
\usepackage{url}

\usepackage{microtype}
\usepackage{graphicx}
\usepackage{float}
\usepackage{booktabs} 
\usepackage{bbding}

\usepackage{hyperref}
\usepackage{array}
\newcolumntype{P}[1]{>{\centering\arraybackslash}p{#1}}

\usepackage{amsmath}
\usepackage{amssymb}
\usepackage{mathtools}
\usepackage{amsthm}
\usepackage[authoryear]{natbib}
\usepackage{enumerate}

\usepackage{caption}
\usepackage{subcaption}
\usepackage{booktabs} 
\usepackage{tikz} 

\usetikzlibrary{shapes.geometric, arrows}

\tikzstyle{startstop} = [rectangle, rounded corners, 
minimum width=1cm, 
minimum height=0.4cm,
text centered, 
draw=black ]

\tikzstyle{io} = [trapezium, 
trapezium stretches=true, 
trapezium left angle=70, 
trapezium right angle=110, 
minimum width=3.0cm, 
minimum height=1cm, text centered, 
draw=black]

\tikzstyle{process} = [rectangle, 
minimum width=3.0cm, 
minimum height=1cm, 
text centered, 
text width=3.0cm, 
draw=black, 
fill]

\tikzstyle{decision} = [diamond, 
minimum width=3.0cm, 
minimum height=1cm, 
text centered, 
draw=black, 
fill=green!30]
\tikzstyle{arrow} = [thick,->,>=stealth]

\usepackage[american]{babel}
\usepackage{amsmath}
\usepackage{amsthm}
\usepackage{amssymb} 
\usepackage[T1]{fontenc} 
\usepackage[capitalise]{cleveref}
\usepackage{ifthen}

\usepackage[ ]{algorithmic}
\usepackage[vlined,ruled, ]{algorithm2e}

\usetikzlibrary{shapes,decorations,arrows,calc,arrows.meta,fit,positioning}
\tikzset{
    -Latex,auto,node distance =1 cm and 1 cm,semithick,
    state/.style ={ellipse, draw, minimum width = 0.7 cm},
    point/.style = {circle, draw, inner sep=0.04cm,fill,node contents={}},
    bidirected/.style={Latex-Latex,dashed},
    el/.style = {inner sep=2pt, align=left, sloped}
}

\newtheorem{theorem}{Theorem}
\newtheorem{corollary}{Corollary}

\newtheorem{lemma}{Lemma}
\newtheorem{remark}{Remark}
\newtheorem{definition}{Definition}
\newtheorem{assumption}{Assumption}

\newcommand{\norm}[1]{\left\lVert #1\right\rVert}
\newcommand{\kl}[2]{\mathrm{KL}(#1 \parallel #2 )}
\newcommand{\KL}[2]{\mathrm{KL}\left(#1 \parallel #2 \right)}
\newcommand{\mcS}{\mathcal{S}}

\newcommand{\abs}[1]{\left|#1\right|}
\newcommand{\until}[1]{\{1,\dots, #1\}}
\newcommand{\tuple}[1]{(#1)}

\newcommand{\supscr}[2]{#1^{\textup{#2}}}
\newcommand{\setdef}[2]{\{#1 \; | \; #2\}}
\newcommand{\seqdef}[2]{\{#1\}_{#2}}
\newcommand{\bigsetdef}[2]{\big\{#1 \; | \; #2\big\}}
\newcommand{\Bigsetdef}[2]{\Big\{#1 \; \Big| \; #2\Big\}}

\newcommand{\tupdef}[2]{(#1)_{#2}}

\newcommand{\union}{\operatorname{\cup}}
\newcommand{\intersection}{\ensuremath{\operatorname{\cap}}}
\newcommand{\indicator}[1]{\mathbf{1}\left\{#1\right\}}

\newcommand{\subject}{\text{subject to}}
\def\checkmark{\tikz\fill[scale=0.4](0,.35) -- (.25,0) -- (1,.7) -- (.25,.15) -- cycle;} 
\DeclareMathOperator*{\argmax}{arg\,max}
\DeclareMathOperator*{\argmin}{arg\,min}
\DeclareMathOperator*{\maximize}{maximize}

\def\prob{\mathbb{P}}
\def\expt{\mathbb{E}}
\def\real{\mathbb{R}}

\def\natural{\mathbb{N}}
\usepackage{tikz}

\usepackage[textsize=tiny]{todonotes}

\usepackage[utf8]{inputenc} 
\usepackage[T1]{fontenc}    
\usepackage{hyperref}       
\usepackage{url}            
\usepackage{booktabs}       
\usepackage{amsfonts}       
\usepackage{nicefrac}       
\usepackage{microtype}      
\usepackage{lipsum}
\usepackage{fancyhdr}       
\usepackage{graphicx}       
\graphicspath{{media/}}     

\usepackage{caption}
\usepackage{subcaption}
\usepackage{comment}

\usepackage{amsmath} 
\usepackage{xspace}

\usepackage{amsmath}
\usepackage{amssymb}
\usepackage{mathtools}
\usepackage{amsthm}
\theoremstyle{plain}

\usepackage{lipsum,lineno}
\usepackage{tabularray}

\usepackage[utf8]{inputenc} 
\usepackage[T1]{fontenc}    
\usepackage{hyperref}       
\usepackage{url}            
\usepackage{booktabs}       
\usepackage{amsfonts}       
\usepackage{nicefrac}       
\usepackage{microtype}      

\usepackage{amsmath}
\usepackage{amssymb}
\usepackage{mathtools}
\usepackage{amsthm}

\usepackage{xargs}                      
\usepackage{xcolor}

\usepackage{tikz}
\usetikzlibrary{positioning, calc, shapes.geometric, shapes, shapes.multipart, arrows.meta, arrows, decorations.markings, external, trees}

\usepackage{scalefnt}
\usetikzlibrary{shapes.misc}
\usetikzlibrary{positioning, calc, shapes.geometric, shapes, shapes.multipart, arrows.meta, arrows, decorations.markings, external, trees, fit}
\tikzset{
	-Latex,auto,node distance =1 cm and 1 cm,semithick,
	state/.style ={ellipse, draw, minimum width = 0.7 cm},
	point/.style = {circle, draw, inner sep=0.04cm,fill,node contents={}},
	nnv/.style={
		rectangle, draw,thick,minimum width=0.7cm,minimum height=1.5cm
	},
        nnh/.style={
            rectangle, draw,thick,minimum width=1.5cm,minimum height=1.0cm
          },
	outer/.style={draw=gray,dashed,thick, inner sep=3pt
	},
	XOR/.style={draw,circle,append after command={
			[shorten >=\pgflinewidth, shorten <=\pgflinewidth,]
			(\tikzlastnode.north) edge (\tikzlastnode.south)
			(\tikzlastnode.east) edge (\tikzlastnode.west)
		}
	},
	bidirected/.style={Latex-Latex,dashed},
	el/.style = {inner sep=2pt, align=left, sloped},
	cross/.style={cross out, draw=black, minimum size=2*(#1-\pgflinewidth), inner sep=0pt, outer sep=0pt},
	cross/.default={1pt}
}

\usepackage{natbib}
\usepackage{microtype}

\title{
Adaptive Online Experimental Design for Causal Discovery 

}

\author{%
  Muhammad Qasim Elahi$^{*1}$,  Lai Wei$^{*\dagger2}$, Murat Kocaoglu$^1$,  Mahsa Ghasemi$^1$\\
  
    School of Electrical and Computer Engineering, Purdue University$^1$\\
      Life Sciences Institute, University of Michigan$^2$\\
  \texttt{elahi0@purdue.edu, weilaitim@gmail.com}\\
  \texttt{mkocaoglu@purdue.edu, mahsa@purdue.edu} \\
}

\usepackage{fancyhdr}
\begin{document}
\def\thefootnote{*}\footnotetext{These authors contributed equally to this work.}\def\thefootnote{\arabic{footnote}}
\def\thefootnote{$\dagger$}\footnotetext{This work started or (was partially conducted) while the author was at Purdue University.}\def\thefootnote{\arabic{footnote}}
\maketitle

\begin{abstract}
Causal discovery aims to uncover cause-and-effect relationships encoded in causal graphs by leveraging observational, interventional data, or their combination. The majority of existing causal discovery methods are developed assuming infinite interventional data. We focus on data interventional efficiency and formalize causal discovery from the perspective of online learning, inspired by pure exploration in bandit problems. A graph separating system, consisting of interventions that cut every edge of the graph at least once, is sufficient for learning causal graphs when infinite interventional data is available, even in the worst case. We propose a track-and-stop causal discovery algorithm that adaptively selects interventions from the graph separating system via allocation matching and learns the causal graph based on sampling history. Given any desired confidence value, the algorithm determines a termination condition and runs until it is met. We analyze the algorithm to establish a problem-dependent upper bound on the expected number of required interventional samples. Our proposed algorithm outperforms existing methods in simulations across various randomly generated causal graphs. It achieves higher accuracy, measured by the structural hamming distance (SHD) between the learned causal graph and the ground truth, with significantly fewer samples.

\end{abstract}

\section{Introduction}
Causal discovery is a fundamental problem encountered across various scientific and engineering disciplines \citep{pearl2009causality, spirtes2000causation, peters2017elements}.  Observational data is generally inadequate for establishing causal relationships and interventional data, obtained by deliberately perturbing the system, becomes necessary. Consequently, contemporary approaches propose leveraging both observational and interventional data for causal discovery \citep{hauser2014two, greenewald2019sample}. A well-established model for depicting causal relationships is based on directed acyclic graphs (DAGs). A directed edge between two variables indicates a direct causal effect, while a directed path indicates an indirect causal effect \citep{spirtes2000causation}.

The causal graph is typically identifiable only up to its Markov equivalence class (MEC) \citep{verma2022equivalence} using observational data. The Markov equivalence class is a set of DAGs that encode the same set of conditional independencies. There is a growing focus on developing algorithms for the design of interventions, specifically aimed at learning causal graphs \citep{hu2014randomized,shanmugam2015learning, ghassami2017optimal}. These algorithms rely on the availability of an infinite amount of interventional data, whose collection in real-world settings is often more challenging and expensive than gathering observational data. In numerous medical contexts, abundant observational clinical data is readily available \citep{subramani1999causal}, whereas conducting randomized controlled trials can be costly or sometimes present ethical challenges. In this work, we consider a scenario where access is limited to only a finite number of interventional samples. Similar to \cite{hu2014randomized,shanmugam2015learning, ghassami2017optimal}, we assume causal sufficiency, meaning that all variables are observed, and no latent or hidden variables are involved.

\begin{table*}[t]
  \centering
  \begin{tabular}{P{5.5cm} P{3.5cm} P{3.5cm} P{3.5cm} }
    \hline
    Reference &  Adaptive/Non-adaptive  & Graph structure constraints&Interventional sample efficiency   \\ \hline
    \citealp{hauser2012characterization} & Non-adaptive & None & \XSolid \\
     \citealp{shanmugam2015learning} & Non-adaptive &   None & \XSolid \\
     \citealp{kocaoglu2017cost} & Non-adaptive & None &  \XSolid \\
     \citealp{greenewald2019sample} & Adaptive &  Trees only & \checkmark \\
     \citealp{squires2020active} & Adaptive & None &  \XSolid \\
    \citealp{choo2023adaptivity} & Adaptive &  None &  \XSolid \\
    
     Track \& Stop Causal Discovery (Ours)  &   Adaptive  & None & \checkmark \\
     \hline
  \end{tabular}
  \caption{A comparison of existing causal discovery techniques with our proposed algorithm}\label{table1_cmp}
  \vspace{-1.7em}
\end{table*}

The PC algorithm \citep{spirtes2000causation} utilizes conditional independence tests in combination with Meek orientation rules \citep{meek1995causal} to recover the causal structure and learn all identifiable causal relations from the data.  The graph separating system, which is a set of interventions that cuts every edge of the graph at least once is sufficient for learning causal graphs when in-
finite interventional data is available, even in the
worst case. \citep{shanmugam2015learning, kocaoglu2017cost}. Bayesian causal discovery is a valuable tool for efficiently learning causal models from limited interventional data, but it encounters challenges when it comes to computing probabilities over the combinatorial space of DAGs \citep{heckerman1997bayesian,annadani2023bayesdag, toth2022active}.  Dealing with the search complexity for DAGs without relying on specific parametric remains a  challenge.

A comparison between our proposed algorithm and existing methods is presented in Table \ref{table1_cmp}. Causal discovery algorithms can be broadly classified into two categories: adaptive and non-adaptive. In the offline setting, interventions are predetermined before algorithm execution. The DAG is learned using corresponding interventional distributions, which require an infinite number of samples \cite{hauser2012characterization,shanmugam2015learning,kocaoglu2017cost}. Contrastingly, existing online discovery  algorithms apply interventions sequentially, with adaptively chosen targets at each step, still necessitating access to an interventional distribution, i.e., an infinite number of interventional samples \cite{squires2020active,choo2023adaptivity}. Although the algorithm by \citeauthor{greenewald2019sample} works with finite interventional data, it is applicable only when the underlying causal structure is a tree. Our proposed tracking and stopping algorithm provides a sample-efficient alternative for general causal graphs.

We approach causal discovery from an online learning standpoint, emphasizing knowledge acquisition and incremental decision-making. Inspired by the pure exploration problem in multi-armed bandit~\citep{kaufmann2016complexity, degenne2019non}, we view the possible interventions as the action space. We propose a discovery algorithm that adaptively selects interventions from the graph separating system via an allocation matching approach similar to one employed in \citealp{wei2024approximate}. Our objective is to uncover the true DAG with a predefined level of confidence while minimizing the number of interventional samples required. The main contributions of our work are listed below:

\vspace{-0.7pt}
\begin{itemize}
\itemsep -0.2em
    \item We  study the causal discovery problem with fixed confidence and proposed a track-and-stop causal discovery algorithm that can adaptively select informative interventions according to the sampling history.
\itemsep -0.1em    
    \item We analyze the algorithm to show it can detect the true DAG with any given confidence level and provide an upper bound on the expected number of required interventional samples.
\itemsep -0.2em    
    \item We conduct a series of experiments using random DAGs and the SACHS Bayesian network from bnlibrary~\citep{scutari2009learning} to compare our algorithm with other baselines. The results show that our algorithm outperforms the baselines, requiring fewer samples.
\end{itemize}

\section{Problem Formulation}

A causal graph $\mathcal{D} = (\mathbf{V}, \mathbf{E})$ is a DAG with the vertex set $\mathbf{V}$ corresponding to a set of random variables. If there is a directed edge $(X, Y) \in \mathbf{E}$ from variable $X$ to variable $Y$, denoted as $X \rightarrow Y$, it means that $X$ is a direct cause or an immediate parent of $Y$. The parent set of a variable $Y$ is denoted by $\mathsf{Pa}(Y)$. The induced graph $\mathcal{D}_\mathbf{X}$ has a vertex set $\mathbf{X}$ and the edge set contains all edges with both endpoints in $\mathbf{X}$. The cut at a set of vertices $\mathbf{X}$, including both oriented and unoriented edges denoted by $E[\mathbf{X}, \mathbf{V} \setminus \mathbf{X}]$, is the set of edges between $\mathbf{X}$ and $\mathbf{V} \setminus \mathbf{X}$. Based on the Markov assumption, the joint distribution can be factorized as $P(\mathbf{v}) = \prod_{i=1}^{n} P(v_i|\mathsf{pa}(X_i))$. 
 A causal graph implies specific conditional independence (CI) relationships among variables through $d$-separation statements. A collection of DAGs is considered Markov equivalent when they exhibit the same set of CI relations. 
\vspace{-1.3pt}
\begin{definition}[Faithfulness \citep{zhang2012strong}]\label{def: faith}
In the population distribution, no conditional independence relations exist other than those implied by the $d$-separation statements in the true causal DAG.
\end{definition}
\vspace{-1.3pt}

 Faithfulness is a commonly used assumption for causal discovery. With the faithfulness assumption, the DAGs in Markov equivalence must share the same skeleton with some edges oriented differently. In order to orient remaining edges, we need access to interventional samples. An intervention on a subset of variables $\mathbf{S} \subseteq \mathbf{V}$, denoted by the do-operator $do(\mathbf{S} = \mathbf{s})$, involves setting each $S_j \in \mathbf{S}$ to $s_j$. Let $\mathcal{D}_{\overline{\mathbf{S}}}$ denote the corresponding interventional causal graph with incoming edges to nodes in $\mathbf{S}$ removed. Using the truncated factorization formula over $\mathcal{D}_{\overline{\mathbf{S}}}$ , if $\mathbf{v}$ is consistent with the realization $\mathbf{s}$, we have:
 \begin{equation}
     P_{\mathbf{s}}(\mathbf{v}):= P(\mathbf{v} \mid do(\mathbf{S} = \mathbf{s})) = \prod_{V_i \notin \mathbf{s}} P(v_i|\mathsf{pa}(V_i))
     \label{trunc_fact}
 \end{equation}
 For a DAG $\mathcal{D}$, we denote the interventional and observational distributions as $P^{\mathcal{D}}_{\mathbf{s}}(\mathbf{v})$ and $P^{\mathcal{D}}(\mathbf{v})$ respectively. In many scenarios, abundant observational data allow for an accurate approximation of the ground truth observational distribution. Therefore, we make the following assumption:

\begin{assumption}
We assume that each variable \( V \in \mathbf{V} \) is discrete and that the observational distribution is available and faithful to the true causal graph.
    \label{assum1}
\end{assumption}

\textbf{Causal Discovery with Fixed Confidence:} Under assumption \ref{assum1}, the causal DAG can be learned up to the MEC with the PC algorithm~\cite{spirtes2000causation}. To orient remaining edges, we need interventional data. We consider a fixed confidence setting, where the learner is given a confidence level $\delta \in (0, 1)$ and is required to output the true DAG with probability at least $1 - \delta$. This problem setup is inspired by the pure exploration problem in multi-armed bandits \citep{kaufmann2016complexity}. It requires the learner to adaptively select informative interventions to reveal the underlying causal structure. With a set of interventional targets $\mathbf{S}$, let the action space be $\mathcal{I} = \bigcup_{\mathbf{S} \in \mathcal{S}} \omega(\mathbf{S})$, where each $\omega(\mathbf{S})$ includes a finite number of interventions $\mathbf{S}$ or its finite number of realizations. The learner sequentially selects intervention $\mathbf{s}_t \in \mathcal{I}$ and observes a sample from the interventional distribution $\mathbf{v}_t \sim P_{\mathbf{s}_t}(\mathbf{v})$. A policy $\pi$ is a sequence $\seqdef{\pi_t}{t \in \natural}$, where each $\pi_t$ determines the probability distribution of taking intervention $\mathbf{s}_t \in \mathcal{I}$ given intervention and observation history $\pi_t(\mathbf{s}_t \mid \mathbf{s}_1, \mathbf{v}_1, \ldots , \mathbf{s}_{t-1}, \mathbf{v}_{t-1})$. 

In a fixed confidence level setting, the number of required interventional samples to output a DAG with a confidence level is unknown beforehand. For a given $\delta \in (0, 1)$, the learner is required to select a stopping time $\tau_{\delta}$ adapted to filtration $ \seqdef{\mathcal{F}_t}{t\in \natural_{>0}}$ where $\mathcal{F}_t = \sigma(\mathbf{s}_1, \mathbf{v}_1, \ldots, \mathbf{s}_{t-1}, \mathbf{v}_{t-1})$. At $\tau_{\delta}$, the learner selects a causal graph based on select rule $\psi(\mathbf{s}_1, \mathbf{v}_1, \ldots, \mathbf{s}_{\tau_{\delta}-1}, \mathbf{v}_{\tau_{\delta}-1})$. The stopping time $\tau_{\delta}$ represents the time when the learner halts and reaches confidence level $\delta$ about a selected causal graph $\psi$. Putting the policy, stopping time, and selection rule together, the triple $\tuple{\pi,\tau_{\delta}, \psi}$ is called a causal discovery algorithm. The objective is to design an algorithm that takes as few interventional samples as possible.

\section{Preliminaries}
In this section, we introduce some foundational concepts and definitions about partially directed graphs, which can be used to encode the MEC and will be employed to develop theoretical results in the paper.

\textbf{Partially Directed Graphs (PDAGs):} A partially directed acyclic graph (PDAG) is a partially directed graph that is free from directed cycles. \citep{perkovic2020identifying}.  A PDAG with some additional edges oriented by the combination of side information and propagation using meek rules is classified as a maximally oriented PDAG (MPDAG).  The unshielded colliders are the variables with two or more parents, where no pair of parents are adjacent. A DAG $\mathcal{D}$ can be represented by an MPDAG or a CPDAG when they share the same set of oriented edges, adjacencies and unshielded colliders. The set of all DAGs represented by the MPDAG $\mathcal{M}$ or a CPDAG $\mathcal{C}$ is denoted by $[\mathcal{M}]$ or $[\mathcal{C}]$, respectively. Both the CPDAGs and MPDAGs take the form of a chain graph with chordal chain components, in which cycles of four or more vertices always contain an additional edge, called a chord \citep{andersson1997characterization}.

 \textbf{Partial Causal Ordering (PCO) in PDAGs :} A path between vertices $X$ and $Y$ is termed a causal path when all edges in the path are directed toward $Y$. A path of the form $P := <V_1 = X, V_2, ..., V_n = Y>$ is categorized as a possibly causal path when it does not contain any edge in the form of $V_i \leftarrow V_j$, where $i < j$. A proper path from $\mathbf{X}$ to $\mathbf{Y}$ is one where only the first node belongs to $\mathbf{X}$ while the remaining nodes do not. If there is a causal path from vertex $x$ to vertex $y$, it implies that $x$ is an ancestor of $y$, i.e., $x \in \mathsf{An}(y)$. Likewise, if there is a possibly causal path from vertex $x$ to vertex $y$, it implies that $x$ is a possible ancestor of $y$, i.e., $x \in \mathsf{PoAn}(y)$. The $\mathsf{An}(\mathbf{X}, \mathcal{M})$ and $\mathsf{PoAn}(\mathbf{X}, \mathcal{M})$ for a set of nodes $\mathbf{X}$ in $\mathcal{M}$ is  the union of over all vertices in $\mathbf{X}$. We adhere to the convention that each node is considered a descendant, ancestor, and possible ancestor of itself.

\begin{definition}
[Partial Causal Ordering] A total ordering of a subset of vertices $\mathbf{X} \subseteq \mathbf{V}$ is a causal ordering of $\mathbf{X}$ in a DAG $\mathcal{D}(\mathbf{V},\mathbf{E})$ if $\;\forall\; X_i,X_j\in \mathbf{X}$ such that $X_i < X_j$ there exists an edge $X_i \rightarrow X_j \in \mathbf{E}$. In the context of an MPDAG, where unoriented edges are present, we can define the Partial Causal Ordering (PCO) of a subset $\mathbf{X} \subseteq \mathbf{V}$ in $\mathcal{M}(\mathbf{V},\mathbf{E})$ as a total ordering of pairwise disjoint subsets $(\mathbf{X}_1, \mathbf{X}_2, ..., \mathbf{X}_m)$ such that $\bigcup_{i=1}^{m} \mathbf{X}_i = \mathbf{X}$. The PCO  must fulfill the following requirement: if $\mathbf{X}_i < \mathbf{X}_j$ and there is an edge between $X_i \in \mathbf{X}_i$ and
$X_j \in \mathbf{X}_j$
in $\mathcal{M}$, then edge $X_i \rightarrow X_j $ is present in  $\mathcal{M}$.
\end{definition}

\section{Algorithm Initialization}
In our problem setup, we assume access to the observational distribution and the corresponding CPDAG \(\mathcal{G}\). We proceed by constructing a graph-separating system for \(\mathcal{G}\) and enumerating possible causal effects.

\subsection{Constructing Graph Separating System}

\begin{definition}[Graph Separating System]
Given a graph \(G = (\mathbf{V}, \mathbf{E})\), a set of different subsets of the vertex set \(V\), \(\mathcal{S} = \{\mathbf{S}_1, \mathbf{S}_2, \ldots, \mathbf{S}_m\}\) is a graph separating system when, for every edge \(\{a, b\} \in E\), there exists a set \(S_i \in \mathcal{S}\) such that either \(a \in S_i\) and \(b \notin S_i\) or \(a \notin S_i\) and \(b \in S_i\).
\end{definition}

In a setting where infinite interventional data is available, the interventional distributions from targets in the graph separating system for unoriented edges in CPDAG $\mathcal{C}$ are necessary and sufficient to learn the true DAG \cite{kocaoglu2017cost, shanmugam2015learning}. Graph coloring, which can be used as a method to generate a separating system by assigning distinct colors to adjacent vertices, is computationally challenging for general graphs. However, for perfect graphs like chordal graphs, efficient polynomial-time algorithms can color the graph using the minimum number of colors \cite{kral2004coloring}. For a set of $n$ variables, a separating system of the form $\mathcal{S} = \{S_1, S_2, ..., S_m\}$, such that $|S_a| \leq k, \; \forall a \in [m]$, is called a $(n, k)$-separating system \citep{katona1966separating, wegener1979separating}. We describe the procedure to generate $(n, k)$-separating system in the supplementary material.

\subsection{Enumerating Causal Effects}
While causal effects are generally not identifiable from CPDAGs or MPDAGs, we can still enumerate all the possible interventional distributions using the causal effect identification formula for MPDAGs \cite{perkovic2020identifying}.

\begin{lemma}
[Causal Identification Formula for MPDAG \citep{perkovic2020identifying}] Consider an MPDAG $\mathcal{M}(\mathbf{V},\mathbf{E})$ and two disjoint sets of variables $\mathbf{X},\mathbf{Y} \subseteq \mathbf{V}$. The interventional distribution $P_{\mathbf{x}}(\mathbf{y})$ is identifiable from any observational distribution consistent with $\mathcal{M}$ if there exists no possibly proper causal path from $\mathbf{X}$ to $\mathbf{Y}$ in $\mathcal{M}$ that starts with an undirected edge and is given below: 
 \begin{equation}
 P_{\mathbf{x}}(\mathbf{y}) =\sum_{\mathbf{b}} \prod_{i=1}^{m} P(\mathbf{b}_i|\mathsf{Pa}(\mathbf{b}_i,\mathcal{M})).
 \end{equation} 
The assignment for $\mathsf{Pa}(\mathbf{b}_i,\mathcal{M})$ must be in consistence with $do(\mathbf{X})$. Also 
$(\mathbf{B}_1,\ldots,\mathbf{B}_m)$ is a partial causal ordering of $\mathsf{An}(\mathbf{Y},\mathcal{M_{\mathbf{V} \setminus \mathbf{X} }})$ in $\mathcal{M}$ and $\mathbf{b} = $ $\mathsf{An}(\mathbf{Y},\mathcal{M_{\mathbf{V} \setminus \mathbf{X} }}) \setminus \mathbf{Y}$. 
\label{iden_thm}
\end{lemma}

\begin{figure}[h]
\centering
\begin{subfigure}[t]{0.20\textwidth}
\centering
    \begin{tikzpicture}
        \node(1) {$V_2$};
        \node(2) [below of = 1,,xshift = -1.0cm,yshift = +0.1cm] {$V_1$};
        \node(3) [below of = 1,,xshift = +1.0cm,yshift = +0.1cm] {$V_3$};
       \node(4) [below of = 1,yshift = -0.8cm] {$V_4$};

   \draw[<-] (2) -- (1);
   \draw[-] (1) -- (3);
   \draw[<-] (2) -- (4);
    \draw[-] (3) -- (4);
   \draw[-] (1) -- (4);

    \end{tikzpicture}
    \subcaption{}
\end{subfigure}
\quad
\begin{subfigure}[t]{0.20\textwidth}
\centering
    \begin{tikzpicture}
     \node(1) {$V_2$};
        \node(2) [below of = 1,,xshift = -1.0cm,yshift = +0.1cm] {$V_1$};
        \node(3) [below of = 1,,xshift = +1.0cm,yshift = +0.1cm] {$V_3$};
       \node(4) [below of = 1,yshift = -0.8cm] {$V_4$};

   \draw[->] (2) -- (1);
   \draw[->] (1) -- (3);
   \draw[<-] (2) -- (4);
    \draw[<-] (3) -- (4);
   \draw[<-] (1) -- (4);

    \end{tikzpicture}
    \subcaption{}
\end{subfigure}
\quad
\centering
\begin{subfigure}[t]{0.20\textwidth}
\centering
    \begin{tikzpicture}
        \node(1) {$V_2$};
        \node(2) [below of = 1,,xshift = -1.0cm,yshift = +0.1cm] {$V_1$};
        \node(3) [below of = 1,,xshift = +1.0cm,yshift = +0.1cm] {$V_3$};
       \node(4) [below of = 1,yshift = -0.8cm] {$V_4$};

   \draw[<-] (2) -- (1);
   \draw[->] (1) -- (3);
   \draw[->] (2) -- (4);
    \draw[<-] (3) -- (4);
   \draw[->] (1) -- (4);
  
    \end{tikzpicture}
    \subcaption{}
\end{subfigure}
\quad
\centering
\begin{subfigure}[t]{0.20\textwidth}
\centering
    \begin{tikzpicture}
      \node(1) {$V_2$};
        \node(2) [below of = 1,,xshift = -1.0cm,yshift = +0.1cm] {$V_1$};
        \node(3) [below of = 1,,xshift = +1.0cm,yshift = +0.1cm] {$V_3$};
       \node(4) [below of = 1,yshift = -0.8cm] {$V_4$};

   \draw[<-] (1) -- (2);
   \draw[->] (1) -- (3);
   \draw[->] (2) -- (4);
    \draw[<-] (3) -- (4);
   \draw[-] (1) -- (4);

    \end{tikzpicture}
    \subcaption{}
\end{subfigure}

\caption{MPDAGs obtained by assigning orientations to edges $E[{V_1},\mathbf{V}\setminus{V_1}]$ in corresponding skeleton i.e. CPDAG}
\label{forbidst}
\end{figure}
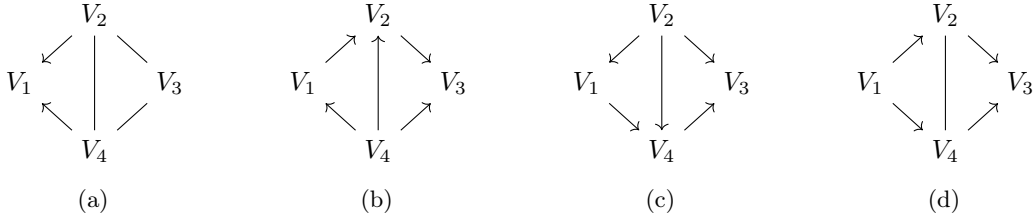

The algorithm for finding the partial causal ordering and enumerating all possible causal effects leveraging~\cref{iden_thm} in an MPDAG is provided in the supplementary material. In cases where one or multiple possibly proper causal paths exist from $\mathbf{X}$ to $\mathbf{Y}$ and start with an undirected edge, the interventional distribution $P_{\mathbf{x}}(\mathbf{y})$ cannot be uniquely determined. However, we can enumerate all possible values for $P_{\mathbf{x}}(\mathbf{y})$ for the set of DAGs represented by the MPDAG $([\mathcal{M}])$. Suppose that $|\mathbf{X}| = k$ and the maximum degree of $\mathcal{M}$ is $d$. This implies that there can be a maximum of $kd$ edges adjacent to the vertices in set $\mathbf{X}$. To enumerate all candidate values of $P_\mathbf{x}(\mathbf{y})$ for every DAG in the set $[\mathcal{M}]$, we assign orientations to all the unoriented edges in $E[\mathbf{X},V \setminus \mathbf{X}]$ and propagate using Meek rules. We denote an orientation of cutting edges $(E[\mathbf{X}, \mathbf{V} \setminus \mathbf{X}])$ as $\mathsf{C}(\mathbf{X})$. This process results in a maximum of $2^{kd}$ partially directed graphs, each one being a valid MPDAG. It's worth noting that, since the first edge of all paths from $\mathbf{X}$ to $\mathbf{Y}$ is oriented, the condition for the identifiability of $P_{\mathbf{x}}(\mathbf{y})$ is satisfied in all of the newly generated MPDAGs. With slight abuse of notation we denote the interventional distribution in the MPDAG $\mathcal{M}$ with cut configuration $\mathsf{C}(\mathbf{X})$ by  $P_{\mathbf{x}}^{\mathsf{C}(\mathbf{X})}$.

Consider a CPDAG $\mathcal{G}$ on the vertex set $\mathbf{V} = \{V_1, V_2, V_3, V_4\}$, which is the same as the complete undirected graph on $\mathbf{V}$ with the edge $V_1 - V_3$ removed. A valid separating set system for $\mathcal{G}$ is $\{\{V_1\}, \{V_1, V_2\}\}$. Figure \ref{forbidst} shows the $4$ possible MPDAGs by assigning different orientations to the cut at $\mathbf{X} =\{V_1\}$. By applying the causal identification formula (\cref{iden_thm}) to all the $4$ MPDAGs, we can identify all possible interventional distributions $P_{\mathbf{x}}(\mathbf{v})$ for MPDAGs in \cref{forbidst} given below.
\begin{equation*} {
   P_{v_1}(v_2,v_3,v_4)   = \begin{cases} 
      P(v_2, v_3, v_4) & (a) \\
      P(v_4)P(v_3|v_2, v_4)P(v_2|v_1, v_4) & (b) \\
      P(v_4|v_1, v_2)P(v_3|v_2, v_4)P(v_2) & (c)\\
     P(v_3|v_2, v_4) P(v_2, v_4|v_1)  & (d)
   \end{cases}}
   \label{enum_caus_e}
\end{equation*}

For example, for the MPDAG in Figure \ref{forbidst}(a), using the algorithm to find PCO, we obtain $\mathsf{PCO}(\mathbf{V} \setminus \mathbf{X}, \mathcal{G}_{\mathbf{V} \setminus \mathbf{X}}) = \{V_2, V_3, V_4\}$, and $P_{\mathbf{x}}(\mathbf{v}) = P(v_2, v_3, v_4)$. We repeat this process for all the MPDAGs in Figure \ref{forbidst} to enumerate all the possible candidate interventional distributions in the above equation. We show in Lemma \ref{our_lemma} that all candidate interventional distributions are different from one another. This implies that we can orient the cutting edges by comparing the candidate interventional distributions with the empirical interventional distribution. In order to ensure that the enumeration step is feasible, we use $(n,k)$ separating system, which implies that for any target set $\mathbf{S}$ we will have at most $2^{kd}$ possible interventional distributions, where $d$ is the maximum degree in the graph. We can then orient the entire DAG by repeating this procedure for all the intervention targets in the separating system.

We define the collection of interventional distributions $P^{\mathcal{D}}_\mathbf{S} = \{P^{\mathcal{D}}_\mathbf{s}\} _{\forall \mathbf{s} \in \text{Dom}(\mathbf{S})}$, where $\text{Dom}(\mathbf{S})$ refers to the domain of $\mathbf{S}$. We show that we have a unique $P^{\mathcal{D}}_\mathbf{S}$ for every possible cutting edge configuration $\mathsf{C}(\mathbf{S})$ in $\mathcal{D}$ (\cref{our_lemma}). The Lemma \ref{our_lemma} implies that there exists a one-to-one mapping between the candidate interventional distributions and the cutting edge orientation $\mathsf{C}(\mathbf{S})$. The proof of \cref{our_lemma} relies on~\citep[Th. 10]{hauser2012characterization}, which requires revisiting some concepts and definitions from the paper \citealp{hauser2012characterization}.

\begin{lemma}
Assume that the faithfulness assumption in~\cref{def: faith} holds and $\mathcal{D}^*$ is the true DAG. For any DAG $\mathcal{D} \neq \mathcal{D^*} $, if $P_{\mathbf{S}}^{\mathcal{D}} = P_{\mathbf{S}}^{\mathcal{D}^*}$   for some $\mathbf{S} \subseteq \mathbf{V}$, they must share the same cutting edge orientation $\mathsf{C}(\mathbf{S})$.
\label{our_lemma}
\end{lemma}
\vspace{-0.2em}

\begin{definition}
    \label{def:intervention-densities}
    Let $\mathcal{D}$ be a DAG on the vertex set $\textbf{V}$, and let $\mathcal{S}$ be a family of targets. Then we define $\mathcal{MK_S}(\mathcal{D})$ as follows:
\begin{equation*}
   \mathcal{MK_S}(\mathcal{D}) = \setdef{(P_{\mathbf{S}}^{\mathcal{D}})_{\mathbf{S} \in \mathcal{S}}}{\text{condition (1) and (2) is true.}}
\end{equation*}
\begin{enumerate}[(1)]
    \item Markov property: $P_{\mathbf{S}}^{\mathcal{D}} \in \mathcal{MK}(\mathcal{D_{\overline{\mathbf{S}} }})$ for all $\mathbf{S}  \in \mcS$.
    \item Local invariance property: for any pair of intervention targets $ \mathbf{S_1}, \mathbf{S_2} \in \mcS,$ for any non-intervened node $U\notin \mathbf{S}_1 \cup \mathbf{S}_2$, $P_{\mathbf{S}_1 }^{\mathcal{D}}(U | \mathsf{Pa}_\mathcal{D}(U)) =P_{\mathbf{S}_2 }^{\mathcal{D}}(U | \mathsf{Pa}_\mathcal{D}(U))$.
\end{enumerate}
  \vspace{-0.4em}
\end{definition}

The $\mathcal{MK_S}(\mathcal{D})$ is space of interventional distribution tuples $(P_{\mathbf{S}}^{\mathcal{D}})_{\mathbf{S} \in \mathcal{S}}$, where each $P_{\mathbf{S}}^{\mathcal{D}}$ is Markov relative to the post-interventional DAG $\mathcal{D}_{\overline{\mathbf{S}}}$, as indicated by $P_{\mathbf{S}}^{\mathcal{D}} \in \mathcal{MK}(\mathcal{D}_{\overline{\mathbf{S}}})$. This suggests that the expression for each $P_\mathbf{s}^{\mathcal{D}}$ can be formulated using truncated factorization over $\mathcal{D}_{\overline{\mathbf{S}}}$ in equation~\eqref{trunc_fact}. Additionally, for any non-intervened variable $U$, the conditional distribution given its parents remains invariant across different interventions. A family of targets $\mathcal{S}$ is considered conservative if, for any $V \in \mathbf{V}$, there exists at least one $ \mathbf{S} \in \mathcal{S}$ such that $V \notin \mathbf{S}$. This implies that any $\mathcal{S}$ containing the empty set, i.e., observational distribution being available, is indeed conservative. Two DAGs $D$ and $D^{*}$ are  $\mcS$-Markov equivalent denoted by $D \sim_\mcS \mathcal{D}^{*}$ if $\mathcal{MK_S}(\mathcal{D}) = \mathcal{MK_S}(\mathcal{D}^{*})$.

\begin{lemma}[\cite{hauser2012characterization}, Th. 10]
    \label{thm:interventional-markov-equivalence}
    Let $\mathcal{D}$ and $\mathcal{D}^{*}$ be two DAGs on $\mathbf{V}$, and $\mcS$  be a conservative family of targets.  Then, the following statements are equivalent: \vspace{-0.8em}
    \begin{enumerate}
        \itemsep -0.15em
        \item \label{itm:interventional-markov-equivalence} $\mathcal{D} \sim_\mcS \mathcal{D}^{*}$.
        \itemsep -0.15em
        \item \label{itm:skeleton-v-structures} $\mathcal{D}$ and $\mathcal{D}^{*}$ have the same skeleton and the same v-structures, and $\mathcal{D}_{\overline{\mathbf{S}}}$ and $\mathcal{D}^{*}_{\overline{\mathbf{S}}}$ have the same skeleton for all $\mathbf{S} \in \mcS$.
       \end{enumerate}
\end{lemma}
\vspace{-1.2em}

\begin{proof}[Proof of \cref{our_lemma}]
    From the definition of interventional markov equivalence, for any two Markov Equivalent DAGs, $\mathcal{D}$ and $\mathcal{D}^{*},$ if they are not $\mathcal{S}$-Markov Equivalent, i.e., $\mathcal{D} \not\sim_{\mathcal{S}} \mathcal{D}^{*},$ this implies $\mathcal{MK}_{\mathcal{S}}(\mathcal{D}) \neq \mathcal{MK}_{\mathcal{S}}(\mathcal{D}^{*}),$ which in turn implies there exists $S \in \mathcal{S}$ such that $P^{\mathcal{D}}_{\mathbf{S}}(\mathbf{v}) \neq P^{\mathcal{D}^{*}}_{\mathbf{S}}(\mathbf{v})$. Also, note that for any set of nodes $\mathbf{S}$, the DAGs with incoming edges to $S$ removed $\mathcal{D}_{\overline{\mathbf{S}}}$ and $\mathcal{D}^{*}_{\overline{\mathbf{S}}}$ share the same skeleton if and only if they have the same cutting edge orientations at $\mathbf{S}$, i.e., $\mathsf{C}(\mathbf{S})$. Now, considering $\mathcal{S} = \{\emptyset, \mathbf{S}\}$ and using \cref{thm:interventional-markov-equivalence}, we have an equivalence relationship between statements $\ref{itm:interventional-markov-equivalence}$ and $\ref{itm:skeleton-v-structures}$, i.e., $\ref{itm:interventional-markov-equivalence} \iff \ref{itm:skeleton-v-structures}$. This equivalence implies that for any two Markov Equivalent DAGs, if they have different cutting-edge configurations $\mathsf{C}(\mathbf{S})$, statement $\ref{itm:skeleton-v-structures}$ does not hold, which, in turn, implies that statement $\ref{itm:interventional-markov-equivalence}$ does not hold. Consequently, $\mathcal{D} \not\sim_{\mathcal{S}} \mathcal{D}^{*}$, suggesting that the joint interventional distribution will differ across the two DAGs, i.e., $P^{\mathcal{D}}_{\mathbf{S}}(\mathbf{v}) \neq P^{\mathcal{D}^{*}}_{\mathbf{S}}(\mathbf{v}).$ The converse of the previous statement, which is that if two Markov equivalent DAGs have the same interventional distribution with some target $\mathbf{S}$, i.e., $P^{\mathcal{D}}_{\mathbf{S}}(\mathbf{v}) = P^{\mathcal{D}^{*}}_{\mathbf{S}}(\mathbf{v})$, they must have the same cutting edge configuration at the target $\mathbf{S}$, is also true.
\end{proof}

\begin{algorithm}[b]
\small
	{	\SetKwInOut{Input}{Input}
		\SetKwInOut{Set}{Set}
		\SetKwInOut{Title}{Algorithm}
		\SetKwInOut{Require}{Require}
		\SetKwInOut{Output}{Output}
            \SetKwInOut{Initialization}{Initialization}
		
		{	
			\Input{CPDAG $\mathcal{C}$, $\delta$, $\mathcal{I}$ and $\tupdef{P_{\mathbf{s}}}{\mathbf{s}\in \mathcal{I}}$ }
			


		\Output{causal discovery result}
		}
		\medskip          
            

  
            select each intervention $\mathbf{s} \in \mathcal{I}$ once
            \smallskip
            
            \While{ $f_t(d_t) \leq \delta $ or $d_t < \abs{\mathcal{I}} (\abs{\omega(V)}-1)$}{

            compute $\boldsymbol{\alpha}_t$ via~\eqref{aleq 4} (or~\eqref{aleq2 2})            
                        

            \nl \If{$\min_{\mathbf{s} \in \mathcal{I}} N_t(\mathbf{s}) < \sqrt{t}$}{
            \emph{\% forced exploration}
            
            select $\mathbf{s}_t = \argmin_{\mathbf{s} \in \mathcal{I}} N_t(\mathbf{s})$  
            } 
            
            \nl \Else{
            
            \emph{\% allocation matching}
            
            select $\mathbf{s}_t = \argmax_{\mathbf{s} \in \mathcal{I}}  \sum_{i=1}^{t} {\alpha}_{\mathbf{s}, i} / N_i(\mathbf{s}) $}
            
            \nl{observe $\mathbf{v}_t$ and update
            
            $N_t(\mathbf{s}_t)\leftarrow N_t(\mathbf{s}_t)+1$, $N_t(\mathbf{s}_t,\mathbf{v}_t)\leftarrow N_t(\mathbf{s}_t, \mathbf{v}_t)+1$
            }
            }
            \textbf{return} $\mathcal{D}_t^*$ in~\eqref{aleq 1} (or $\tupdef{\mathsf{C}^*_t(\mathbf{S})}{\mathbf{S} \in \mathcal{S}}$ in~\eqref{aleq2 1})
              
        } 
		
		\caption{Track-and-stop Causal Discovery}
          
		\label{alg: TSCD}
\end{algorithm} 

\section{Online Algorithm Design and Analysis }
We design a data-efficient causal discovery algorithm. After initialization, the CPDAG, a graph separating system, and all possible causal effects are available. We proceed to propose a track-and-stop causal discovery algorithm that adaptively selects informative interventions. We analyze it to show it can discover the true DAG with any given confidence level $1- \delta$ for any $\delta \in (0,1)$. In casual discovery, reaching a confidence level $1 - \delta$ itself is not a challenging task since the learner can take arbitrarily many interventional samples. The overarching objective is to minimize the number of interventions required to reach the accuracy level $\tau_\delta$. Since the stopping time $\tau_{\delta}$ is random, in fact, $\expt[\tau_{\delta}]$ is minimized. A sound algorithm needs to be instance-dependent, which means it is capable of detecting any DAG $\mathcal{D^*} \in [\mathcal{C}]$ if it is the ground truth. Also in line with the definition of stopping times, for a poorly designed algorithm, it is possible that $\tau_{\delta} = \infty$, which means the learner can never make a decision. Bringing both aspects together, a sound causal discovery algorithm is formally defined as follows.

\begin{definition}[Soundness of Algorithm]~\label{def: soundness}
    For a given confidence level $\delta \in (0, 1)$, a causal discovery algorithm $\tuple{\pi,\tau_{\delta}, \psi}$ is sound if for any $\mathcal{D}^* \in [\mathcal{C}]$, it satisfies
    \[\prob(\tau_{\delta} < \infty, \psi = \mathcal{D}^*) \geq 1 - \delta.\]
\end{definition}
The following theorem gives a lower bound on $\expt[\tau_\delta]$ for all sound algorithms to discover the true DAG, which serves as the ultimate target we follow in algorithm design. It has a similar form to the sampling complexity of the bandit problem, whose objective is to identify the optimal arm. The proof follows a similar procedure as~\cite{kaufmann2016complexity}, and we defer its proof to the appendix. 

\begin{theorem}\label{th: lowerbound}
    For the causal discovery problem, suppose the MEC represented by CPDAG $\mathcal{C}$ and observational distributions are available. Assume that $\tuple{\pi,\tau_{\delta}, \psi}$ is sound for $\mathcal{D}^*$ at confidence level $\delta\in(0,1)$. It holds that $\expt[ \tau_{\delta} ] \geq  \log({4}/{\delta}) / c(\mathcal{D}^*)$, where
    \begin{equation}\label{eq: lowerbound}
        c(\mathcal{D}^*) = \sup_{\boldsymbol{\alpha} \in \Delta(\mathcal{I})} \min_{ \mathcal{D} \in [\mathcal{C}] \setminus \mathcal{D}^* } \sum_{\mathbf{s} \in \mathcal{I}}  \alpha_{\mathbf{s}} \KL{P^{\mathcal{D}^*}_{\mathbf{s}}}{P^{\mathcal{D}}_{\mathbf{s}}},
    \end{equation}
and $\Delta(\mathcal{I}):= \bigsetdef{\boldsymbol{\alpha} \in \real_{\geq 0}^{\abs{\mathcal{I}}}}{ \sum_{\mathbf{s} \in \mathcal{I}} \alpha_{\mathbf{s}} = 1}$.
\end{theorem}

The lower bound can be interpreted as follows. By mixing up interventions in $\mathcal{I}$ with oracle allocation $\boldsymbol{\alpha}$, the average information distance generated from $\mathcal{D}^*$ to $\mathcal{D}$ is $\sum_{\mathbf{s} \in \mathcal{I}}  \alpha_{\mathbf{s}} \kl{P^{\mathcal{D}^*}_{\mathbf{s}}}{P^{\mathcal{D}}_{\mathbf{s}}}$. To identify the true DAG with probability at least $1-\delta$,~\cref{th: lowerbound} suggest at least $\log(4/\delta )$ information distance is required to be generated from $\mathcal{D}^*$ to any other $\mathcal{D} \in [\mathcal{C}] \setminus \mathcal{D}^*$, which explains the minimization term in~\eqref{eq: lowerbound}. The parameter $\alpha$ that solves~\eqref{eq: lowerbound} suggests an optimal allocation of interventions in $\mathcal{I}$. However, computing $\alpha$ requires true interventional distribution $P_{\mathbf{s}}$ for each $\mathbf{s} \in \mathcal{I}$. The allocation matching principle essentially replaces the true interventional distributions with an estimated one to compute $\alpha$ and select samples to match it. The key idea will be elaborated in the upcoming algorithm design section.

\subsection{The Exact Algorithm}
In this section, we propose the track-and-stop causal discovery whose pseudo-code is shown in Algorithm~\ref{alg: TSCD}. It is asymptotically optimal as it achieves the $O(\log(1/{\delta}) / c(\mathcal{D}^*))$ lower bound in~\cref{th: lowerbound} on the expected number of required interventions. However, it is computationally intense. In the next section, we propose its practical implementation that reduces computational complexity at the cost of acceptable reduced efficiency.

\textbf{Tracking and Termination Condition:} Let $N_t(\mathbf{s})$ be the number of intervention $do(\mathbf{S} = \mathbf{s})$ taken till $t$, and let $N_t(\mathbf{s}, \mathbf{v})$ be the number of times $\mathbf{v}$ is observed by taking intervention $do(\mathbf{S} = \mathbf{s})$. The most probable DAG can be computed as
\begin{equation}\label{aleq 1}
    \mathcal{D}^*_t \in \argmax_{\mathcal{D} \in [\mathcal{C}]} \sum_{\mathbf{s} \in \mathcal{I}}  N_t(\mathbf{s}, \mathbf{v}) \log P^{\mathcal{D}}_{\mathbf{s}}(\mathbf{v}),
\end{equation}
where $P^{\mathcal{D}}_{\mathbf{s}}$ can be computed based on the configuration of cutting edges of $\mathbf{S}$ in $\mathcal{D}$ according to~\cref{iden_thm}. Let $\Bar{P}_{\mathbf{s}, t} (\mathbf{v}) = N_t(\mathbf{v}, \mathbf{s}) / N_t(\mathbf{s})$ be the empirical distribution conditioned on taking intervention $do(\mathbf{S} = \mathbf{s})$. To evaluate if $\mathcal{D}^*_t$ has reached the confidence level $1-\delta$, we compute 
\begin{equation}\label{aleq 2}
    d_t =  \min_{\mathcal{D} \in [\mathcal{C}]\setminus \mathcal{D}^*_t} \sum_{\mathbf{s} \in \mathcal{I}} N_t(\mathbf{s})  \kl{\Bar{P}_{\mathbf{s}, t}}{P^{\mathcal{D}}_{\mathbf{s}}},
\end{equation}
which is the cumulative information distance between the empirical distribution and interventional distribution from the second most probable DAG. The algorithm terminates if $f_t(d_t) < \delta$, where 
\begin{equation}
    f_t(x)=\bigg (\frac{x \lceil x \ln t  + 1 \rceil 2 e}{\abs{\mathcal{I}} (\abs{\omega(V)}-1)}\bigg)^{\abs{\mathcal{I}} (\abs{\omega(V)}-1)}  e^{1 - x},
\end{equation}
and returns $\mathcal{D}^*_t$. The function $f_t(x)$ is selected according to a concentration bound for Categorical distributions~\cite{van2020optimal}, and it guarantees the probability of $\mathcal{D}^*_t \neq \mathcal{D}^*$ to be lower than $\delta$.

\textbf{Intervention Selection Rule:}
Inspired by~\cref{th: lowerbound}, we intend to design an efficient causal discovery strategy such that $N_t(\mathbf{s}) \approx \alpha_{\mathbf{s}} t$ for each $\mathbf{s}\in\mathcal{I}$. Since ground truth $(P_{\mathbf{s}}^*)_{\mathbf{s}\in\mathcal{I}}$ is unavailable, at each time $t$, we use $(\Bar{P}_{\mathbf{s},t})_{\mathbf{s}\in\mathcal{I}}$ instead to solve for $\alpha_t$ to approximate the oracle allocation $\boldsymbol{\alpha}$. To make this approach work, we need to ensure every intervention is taken a sufficient amount of times so that each $\Bar{P}_{\mathbf{s},t}$ converges to the $P^*_{\mathbf{s}}$ in a fast enough rate. Accordingly, if $\min_{\mathbf{s} \in \mathcal{I}} N_t(\mathbf{s}) \leq \sqrt{t}$, the forced exploration step selects the least selected intervention so that it guarantees that each intervention is selected at least $\Omega(\sqrt{t})$ times. 

To solve for the sequence $\seqdef{\boldsymbol{\alpha}_t}{t=1}^{T}$, we substitute $\Bar{P}_{\mathbf{s},t}$ and $\mathcal{D}_t^*$ into~\eqref{eq: lowerbound} and take an online optimization procedure to
\begin{equation}\label{aleq 3}
    \maximize_{\forall t: \boldsymbol{\alpha}_t \in \Delta(\mathcal{I})} \sum_{t = 1}^T \min_{\mathcal{D} \in [\mathcal{C}]\setminus \mathcal{D}^*_t} \sum_{\mathbf{s} \in \mathcal{I}} \alpha_{\mathbf{s}, t} \kl{\Bar{P}_{\mathbf{s}, t}}{P^{\mathcal{D}}_{\mathbf{s}}}.
\end{equation}
Let $\mathcal{D}'_t \in \argmin_{\mathcal{D} \in [\mathcal{M}]\setminus \mathcal{D}^*_t} \sum_{\mathbf{s} \in \mathcal{I}} \alpha_{\mathbf{s}, t} \kl{\Bar{P}_{\mathbf{s}, t}}{P^{\mathcal{D}}_{\mathbf{s}}}$ and $\boldsymbol{r}_t \in \real^{\abs{\mathcal{I}}}$ be a vector with entries $r_{\mathbf{s},t} = \kl{\Bar{P}_{\mathbf{s}, t}}{P^{\mathcal{D}'_t}_{\mathbf{s}}}$. Note that $\boldsymbol{\alpha}_t \in \Delta(\mathcal{I})$ for all $t$. We follow the AdaHedge algorithm~\citep{de2014follow} to set
\begin{equation}\label{aleq 4}
    \alpha_{\mathbf{\mathbf{s}}, 1} = \frac{1}{\abs{\mathcal{I}}}, \,\,\alpha_{\mathbf{s},t+1} = \frac{\alpha_{\mathbf{s},t} e^{\eta_t r_{\mathbf{s},t} }}{\sum_{{\mathbf{s} \in \mathcal{I}}} \alpha_{\mathbf{s},t} e^{\eta_t r_{\mathbf{s},t} }} , \forall{\mathbf{s} \in \mathcal{I}},
\end{equation}
where $\eta_t$ is a decreasing learning rate with update rule
\[\eta_{t+1} = \frac{\ln K}{\Delta_{t}}, \,\, \Delta_{t} = \sum_{i=1}^t \frac{1}{\eta_t} \ln \langle \boldsymbol{\alpha}_t, e^{\eta_t \boldsymbol{r}_t}\rangle - \sum_{\mathbf{s} \in \mathcal{I}} \alpha_{\mathbf{s}, t} r_{\mathbf{s},t}.\]
To make $N_t(\mathbf{s})$ track $\sum_{i=1}^t \alpha_{\mathbf{s}}$, the allocation matching step selects $\argmax_{\mathbf{s} \in \mathcal{I}}  \sum_{i=1}^t {\alpha}_{\mathbf{s}, i} / N_t(\mathbf{s})$.
\begin{remark}
    The proposed algorithm is computationally intensive since in equations~\eqref{aleq 1}~\eqref{aleq 2} and~\eqref{aleq 3}, it needs to enumerate DAGs in $[\mathcal{C}]$ which can be exponentially many. The worst-case computational complexity can be $\Omega(2^n)$, where $n$ is the number of unoriented edges in $\mathcal{C}$.
\end{remark}

\subsection{Practical Algorithm Implementation}
In a practical implementation of track-and-stop causal discovery, we treat learning the configuration of edge cut $\mathsf{C}^*(\mathbf{S})$ for each node set $\mathbf{S} \in \mathcal{S}$ as an individual task, and apply a local learning strategy. The global strategy assigns allocation according to feedback from local learning results.

\textbf{Local strategy:} With~\cref{our_lemma}, the intervention $\mathbf{S} \in \mathcal{S}$ is sufficient to learn the edge cut corresponding to $\mathbf{S}$. At time $t$, we compute a local allocation rule $\boldsymbol{\xi}^{\mathbf{S}}_{t} \in \Delta(\omega(\mathbf{S}))$ to learn the edge cut of $\mathbf{S}$. Let the most probable configuration of edge cut be computed as
\begin{equation}\label{aleq2 1}
    \mathsf{C}^*_t(\mathbf{S}) = \argmax_{\mathsf{C}(\mathbf{S})} \sum_{\mathbf{s} \in \omega(\mathbf{S})}  N_t(\mathbf{s}, \mathbf{v}) \log P^{\mathsf{C}(\mathbf{S})}_{\mathbf{s}} (\mathbf{v}).
\end{equation}
Similar to~\eqref{aleq 3}, we solve $\xi_{\mathbf{S},t}$ via online optimization
\[\maximize_{\forall t: \boldsymbol{\xi}^{\mathbf{S}}_t \in \Delta(\omega(\mathbf{S}))} \sum_{t = 1}^T \min_{\mathsf{C}(\mathbf{S}) \neq \mathsf{C}_t^*(\mathbf{S})} \sum_{\mathbf{s} \in \omega(\mathbf{S})} \xi^{\mathbf{S}}_{\mathbf{s}, t} \kl{\Bar{P}_{\mathbf{s}, t}}{P^{\mathsf{C}(\mathbf{S})}_{\mathbf{s}}}.\]

The update rule for $\boldsymbol{\xi}^\mathbf{S}_{t} \!\!\!\in \!\!\Delta(\omega(\mathbf{S}))$ is similar to~\eqref{aleq 4}. Let
 $\mathsf{C}'_t(\mathbf{S}) \!\! \in \!\!\argmin_{\mathsf{C}(\mathbf{S}) \neq \mathsf{C}_t^*(\mathbf{S})} \sum_{\mathbf{s} \in \omega(\mathbf{S})} \xi^{\mathbf{S}}_{\mathbf{s}, t} \kl{\Bar{P}_{\mathbf{s}, t}}{P^{\mathsf{C}(\mathbf{S})}_{\mathbf{s}}}$ and define vector $\boldsymbol{r}_t^{\mathbf{S}} \in \real^{\abs{\omega(\mathbf{S})}}$ with each entry to be $r_{\mathbf{s},t}^{\mathbf{S}} = \kl{\Bar{P}_{\mathbf{s}, t}}{P^{\mathsf{C}'_t(\mathbf{S}) }_{\mathbf{s}}}$. Then we set
\begin{equation*}
    \xi^{\mathbf{S}}_{\mathbf{\mathbf{s}}, 1} = \frac{1}{\abs{\omega(\mathbf{S})}}, \,\,\xi^{\mathbf{S}}_{\mathbf{s},t+1} = \frac{\xi^{\mathbf{S}}_{\mathbf{s},t} e^{\eta_t r_{\mathbf{s},t} }}{\sum_{\mathbf{s} \in \omega(\mathbf{S})}\xi^{\mathbf{S}}_{\mathbf{s},t} e^{\eta_t r_{\mathbf{s},t} }} , \forall{\mathbf{s} \in \omega(\mathbf{S})},
\end{equation*}
where $\eta_{t+1} = {\ln K}/{\Delta_{t}}$ and
\[\Delta_{t} = \sum_{i=1}^t \frac{1}{\eta_t} \ln \langle \boldsymbol{\xi}^{\mathbf{S}}_{t}, e^{\eta_t \boldsymbol{r}_t(\mathbf{S})}\rangle - \sum_{\mathbf{s} \in \omega(\mathbf{S})} \xi^{\mathbf{S}}_{\mathbf{s}, t} r_{\mathbf{s},t}^{\mathbf{S}}.\]

\textbf{Global Strategy: }
To allocate interventions on different node sets in $\mathcal{S}$, we design a global allocation strategy $\boldsymbol{\gamma}_t \in \Delta(\mathcal{S})$ at each step. Taking feedback from $\abs{\mathcal{S}}$ local strategies, let $c_t (\mathbf{S}) = \frac{1}{t} \sum_{i=1}^t \sum_{\mathbf{s} \in \omega(\mathbf{S})} \xi^{\mathbf{S}}_{\mathbf{s}, t}  r_{\mathbf{s},t}^{\mathbf{S}}$. The value $1/c_t (\mathbf{S})$ corresponds to the estimated difficulty of learning the edge cut of $\mathbf{S}$. Accordingly, set $\gamma_{\mathbf{S},t} = \frac{{1}/{c_t (\mathbf{S})}}{\sum_{\mathbf{S} \in \mathcal{S}} {1}/{c_t (\mathbf{S})}} $ and let
\begin{equation}\label{aleq2 2}
    \alpha_{\mathbf{s}} = \gamma_{\mathbf{S},t} \xi^{\mathbf{S}}_{\mathbf{s},t} \, \forall \mathbf{S} \in \mathcal{I}.
\end{equation}

\textbf{Tracking and Termination: }
The algorithm keeps track of $\tupdef{\mathsf{C}^*_t(\mathbf{S})}{\mathbf{S} \in \mathcal{S}}$ as the candidate causal discovery result. To evaluate if the confidence level $\delta$ is reached about $ \tupdef{\mathsf{C}^*_t(\mathbf{S})}{\mathbf{S} \in \mathcal{S}}$, for each $\mathbf{S}\in \mathcal{S}$, let
\[Z_t(\mathbf{S})=  \min_{\mathsf{C}(\mathbf{S}) \neq \mathsf{C}^*_t(\mathbf{S})} \sum_{\mathbf{s} \in \omega(\mathbf{S})} N_t(\mathbf{s})  \kl{P^{\mathsf{C}(\mathbf{S})}_{\mathbf{s}, t}}{P^{\mathsf{C}^*_t(\mathbf{S})}_{\mathbf{s}}},\]
which is the minimal additional information distance by changing the edge cut of $\mathbf{S}$. Then, we set $d_t$ to be
\[d_t = \min_{\mathbf{S}\in\mathcal{S}} Z_t (\mathbf{S}) + \sum_{\mathbf{S}\in\mathcal{S}} \sum_{\mathbf{s} \in \omega(\mathbf{S})}  N_t(\mathbf{s}) \kl{\Bar{P}_{\mathbf{s}, t}}{P^{\mathsf{C}^*_t(\mathbf{S})}_{\mathbf{s}}}.\]
If $f_t(d_t) < \delta$, the algorithm stops and returns $\tupdef{\mathsf{C}^*_t(\mathbf{S})}{\mathbf{S} \in \mathcal{S}}$. Other aspects of the algorithm remains unchanged.
\begin{remark}
    Instead of enumerating all DAGs in $[\mathcal{C}]$, the practical implementation enumerates configurations of cutting edges for each $\mathbf{S} \in \mathcal{S}$. It is possible to output $\tupdef{\mathsf{C}_t(\mathbf{S})}{\mathbf{S} \in \mathcal{S}}$ with contradictory edge orientations or violation of the DAG criteria. But the overall probability of $\tupdef{\mathsf{C}^*_{\tau_{\delta}}(\mathbf{S})}{\mathbf{S} \in \mathcal{S}}$ not matching with the true DAG $\mathcal{D}^*$ is bounded by $\delta$. If $\tupdef{\mathsf{C}^*_{\tau_{\delta}}(\mathbf{S})}{\mathbf{S} \in \mathcal{S}}$ is not a DAG, it is suggested to reduce $\delta$ and continue the causal discovery experiment.
\end{remark}

\subsection{An Asymptotic Analysis of Algorithm}
Let $\mathcal{A}_{\mathsf{I}}$ and $\mathcal{A}_{\mathsf{P}}$ denote the exact track-and-stop causal discovery algorithm and its practical implementation, respectively. To characterize the performance for $\mathcal{A}_{\mathsf{P}}$, we define 
\[\underline{c}(\mathcal{D}^*) \!:= \!\!\!\sup_{\boldsymbol{\alpha} \in \Delta(\mathcal{I})} \min_{\mathbf{S} \in \mathcal{S}} \min_{\mathsf{C}(\mathbf{S}) \neq \mathsf{C}^*(\mathbf{S})} \! \sum_{\mathbf{s} \in \omega(\mathbf{S})} \!\! \alpha_{\mathbf{s}} \kl{P^{\mathcal{D}^*\!\!}_{\mathbf{s}}}{P^{\mathsf{C}(\mathbf{S})}_{\mathbf{s}}},\]
which is a lower bound for $c(\mathcal{D}^*)$ in~\eqref{eq: lowerbound}.
\begin{theorem} \label{th: upperbound}
For the causal discovery problem, suppose the MEC represented by CPDAG $\mathcal{C}$ and observational distributions are available. If the faithfulness assumption in~\cref{def: faith} holds, for both $\mathcal{A}_{\mathsf{I}}$ and $\mathcal{A}_{\mathsf{P}}$,
\begin{itemize}
    \item $\prob(\psi \neq \mathcal{D}^*) \leq \delta$ and $\prob(\tau_{\delta} = \infty) = 0$.
    \item The expected number of required interventions
    \[\lim_{\delta \rightarrow 0} \frac{\log (1/\delta)}{ \expt[\tau_{\delta}]} = \begin{cases}
        {c}(\mathcal{D}^*), & \mathcal{A}_{\mathsf{I}}, \\
        \underline{c}(\mathcal{D}^*), & \mathcal{A}_{\mathsf{P}},\\
    \end{cases}\]
    where ${c}(\mathcal{D}^*) \geq \underline{c}(\mathcal{D}^*)$.
\end{itemize}    
\end{theorem}
The output of $\psi$ can be either $D_{\tau^*_\delta}$ or $\tupdef{\mathsf{C}_{\tau_{\delta}}(\mathbf{S})}{\mathbf{S} \in \mathcal{S}}$, and $\psi \neq \mathcal{D}^*$ in general means that the output does not match $\mathcal{D}^*$. The theorem shows that both $\mathcal{A}_{\mathsf{I}}$ and $\mathcal{A}_{\mathsf{P}}$ are sound, and $\mathcal{A}_{\mathsf{I}}$ archives a asymptotic performance matching with the lower bound in~\cref{th: lowerbound}.

\begin{figure*}[b!]
\centering
\begin{subfigure}[b]{4.6cm}
    \includegraphics[height = 3.0cm,width=4.6cm]{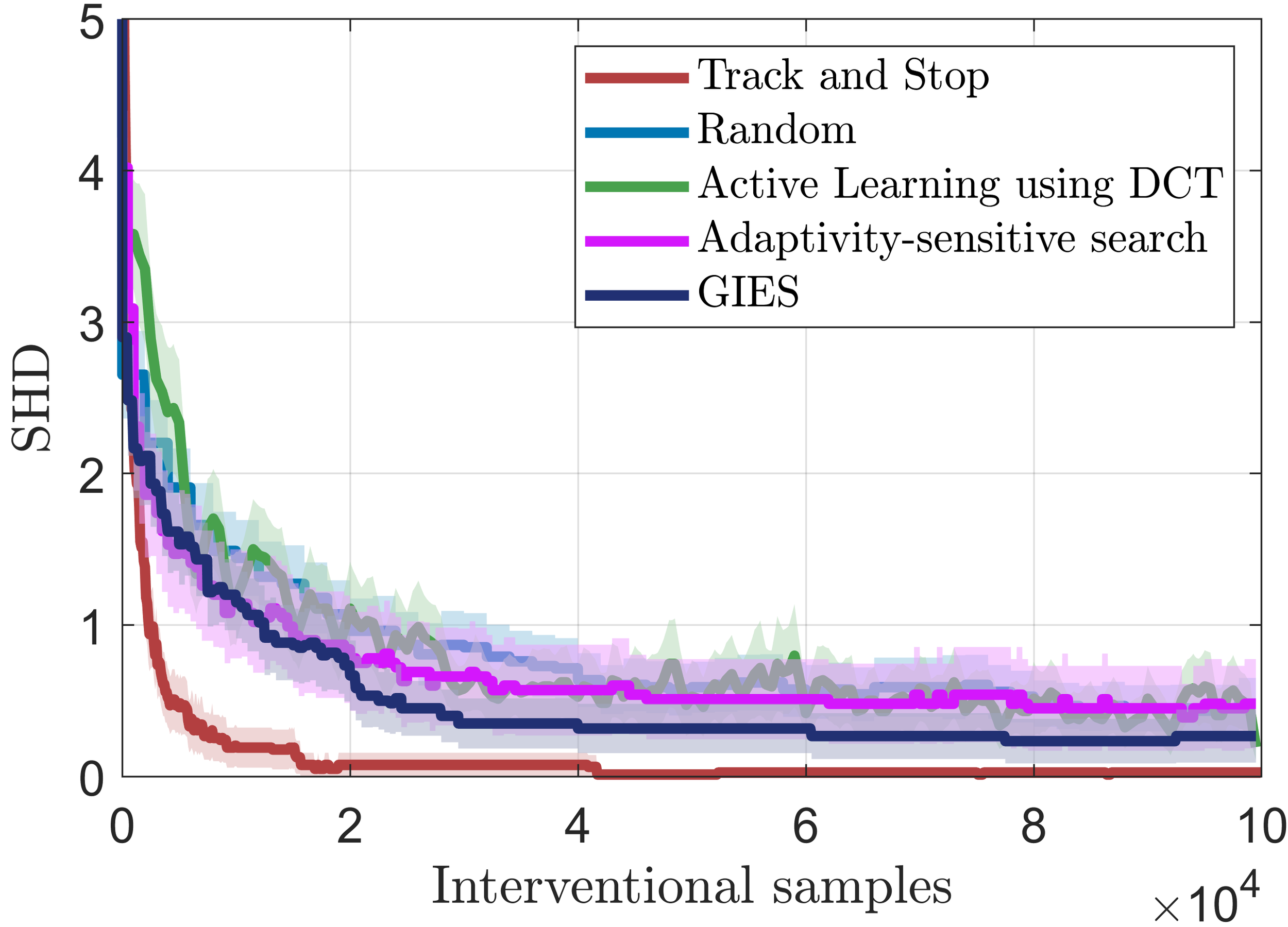}
    \subcaption{$N=5,\rho = 1$}
    \end{subfigure}
\enspace
\begin{subfigure}[b]{4.6cm}
    \includegraphics[height = 3.0cm,width=4.6cm]{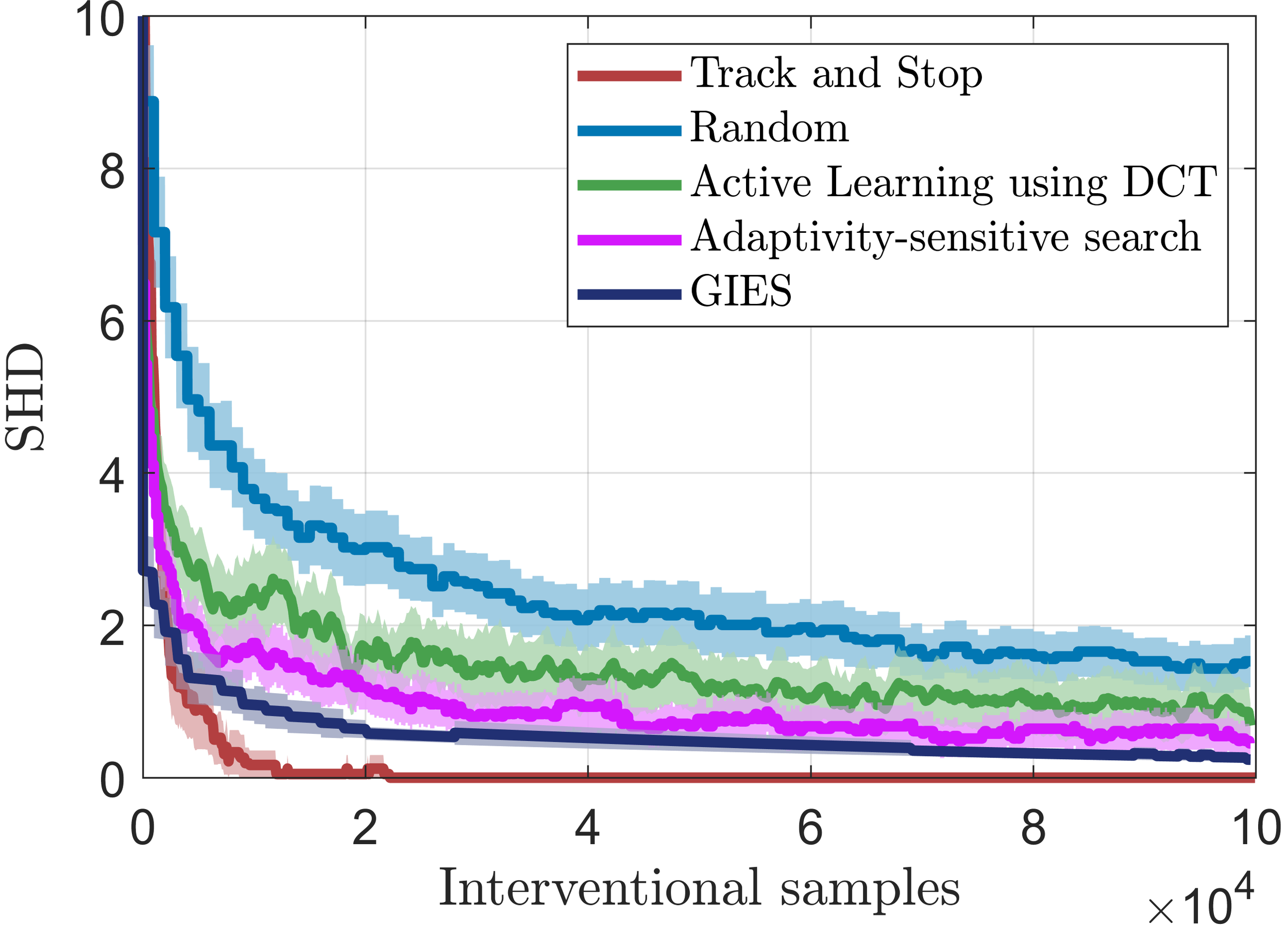}
  \subcaption{$N=6,\rho = 1$}
    \end{subfigure}
\enspace
\begin{subfigure}[b]{4.6cm}
    \includegraphics[height = 3.0cm,width=4.6cm]{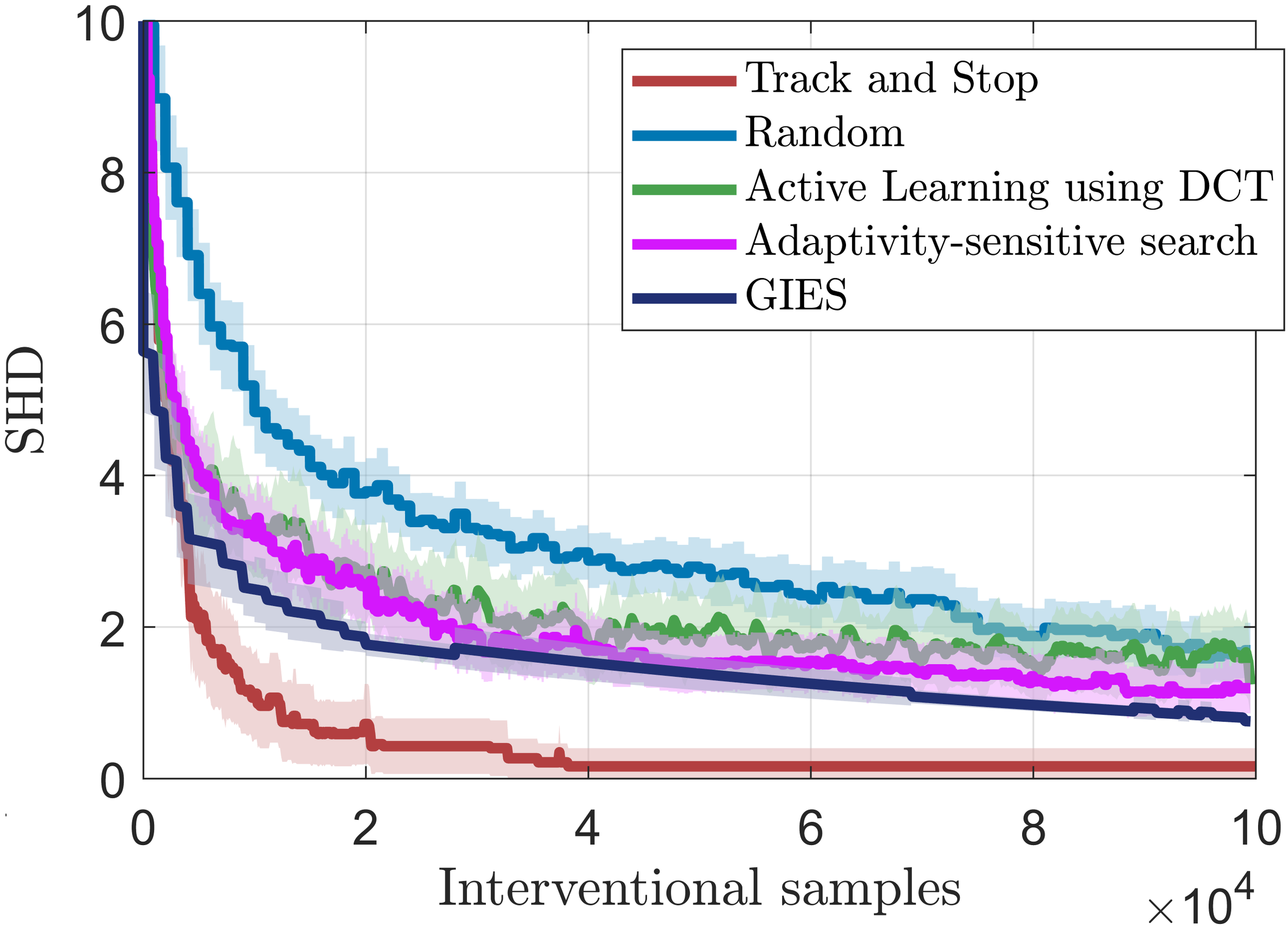}
        \subcaption{ $N=7,\rho = 1$}
\end{subfigure}     
\caption{SHD vs interventional samples for complete Erdös-Rényi random chordal graphs with varying graph orders.}
\label{syn_res1}
\end{figure*}

\begin{figure*}[b!]
\centering
\begin{subfigure}[b]{4.6cm}
    \includegraphics[height = 3.0cm,width=4.6cm]{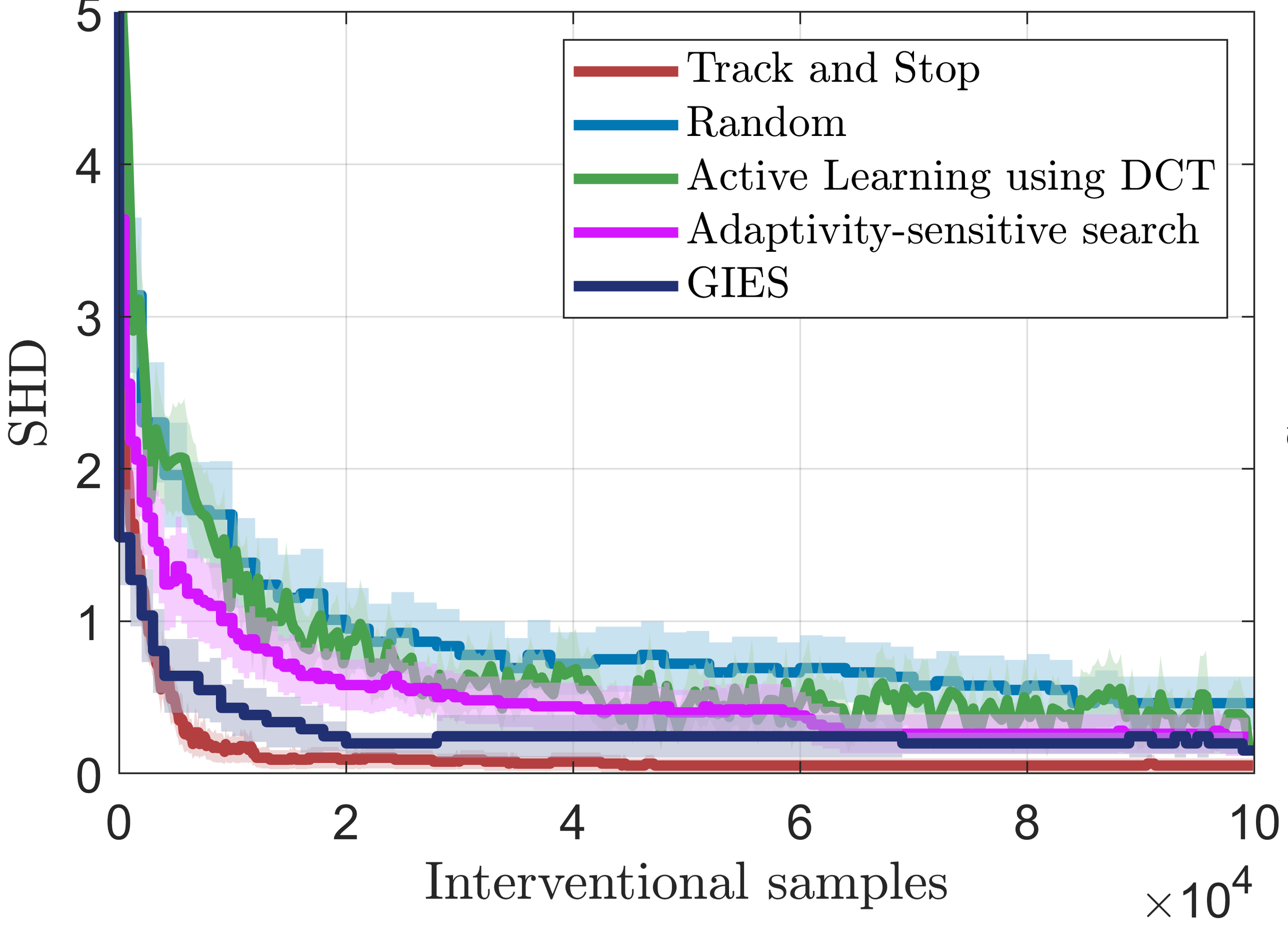}
    \subcaption{$N=10,\rho = 0.1$}
    \end{subfigure}
\enspace
\begin{subfigure}[b]{4.6cm}
    \includegraphics[height = 3.0cm,width=4.6cm]{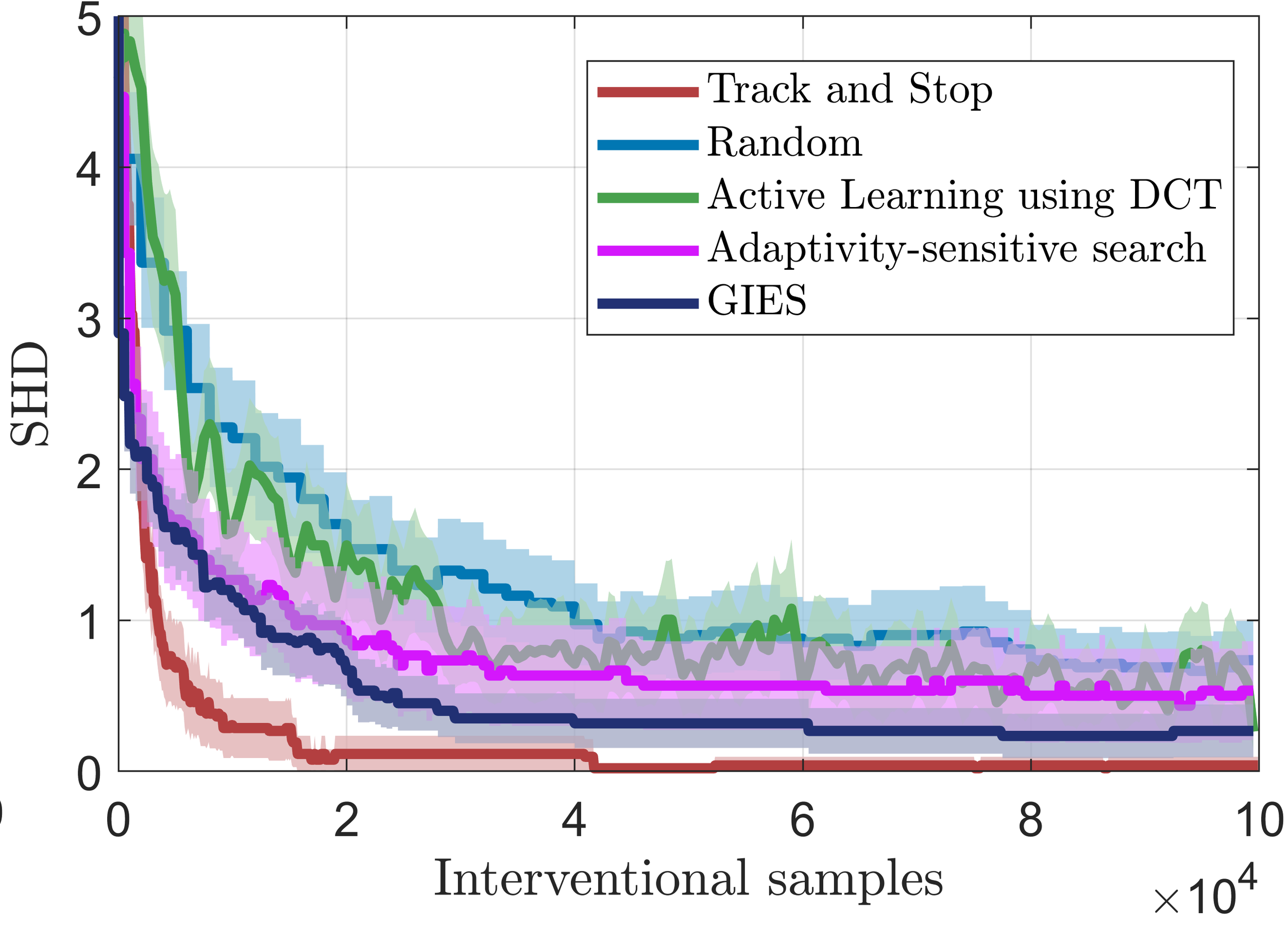}
  \subcaption{$N=10,\rho = 0.15$}
    \end{subfigure}
\enspace
\begin{subfigure}[b]{4.6cm}
    \includegraphics[height = 3.0cm,width=4.6cm]{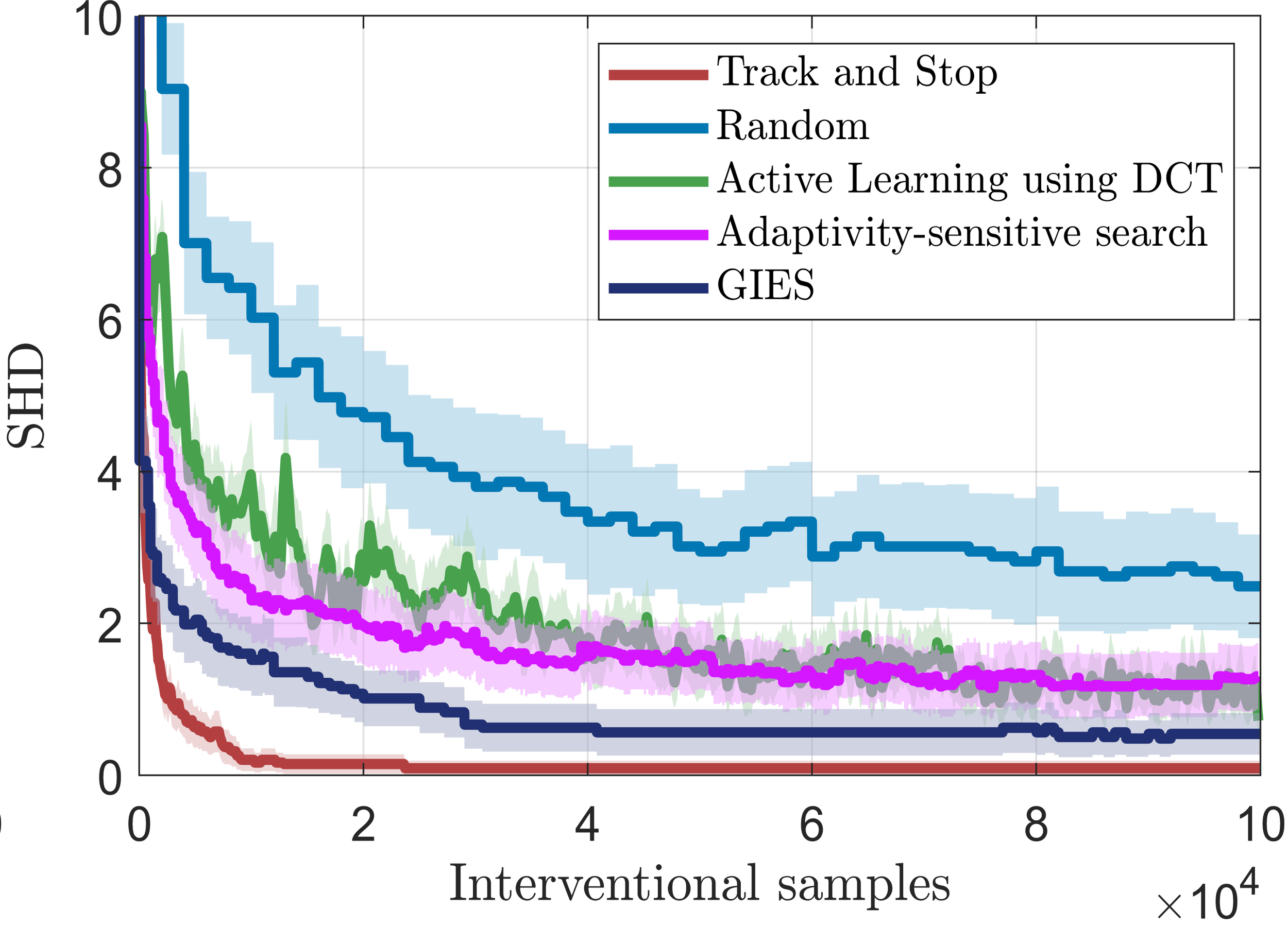}
        \subcaption{ $N=10,\rho = 0.2$}
\end{subfigure}     
\caption{SHD vs interventional samples for Erdös-Rényi random chordal graphs with varying graph density.}

\label{syn_res2}
\end{figure*}

\section{Experiments}
We compare the proposed track-and-stop causal discovery algorithm with four other baselines. The first baseline consists of random interventions within the graph separating system. For each time step, only one sample is collected, and independence tests are used to learn the cuts at the targets based on the available samples from each intervention target within the graph separating system. The second baseline employs Active Structure Learning of Causal DAGs via Directed Clique Trees (DCTs) \citep{squires2020active}. The third one is the adaptive sensitivity search algorithm proposed in \cite{choo2023adaptivity}. The fourth baseline is Greedy Interventional Equivalence Search (GIES) which is used for regularized maximum likelihood estimation in an interventional setting \citep{hauser2012characterization}.

We randomly sample connected moral DAGs using a modified Erdös-Rényi sampling approach. Initially, we generate a random ordering $\sigma$ over vertices. Subsequently, for the n\textsuperscript{th} node, we sample its in-degree as $X_n = \max(1, \text{Bin}(n - 1, \rho))$ and select its parents by uniformly sampling from the nodes that precede it in the ordering. In the final step, we chordalize the graph by applying the elimination algorithm \citep{koller2009probabilistic}, using an elimination ordering that is the reverse of $\sigma$. This procedure  is similar to the one used by \cite{squires2020active}. Finally, we randomly sample the conditional probability tables (CPTs) consistent with the sampled DAG and run the causal discovery algorithms.

\begin{figure}[t!]
    \centering
    \includegraphics[width=0.33\textwidth]{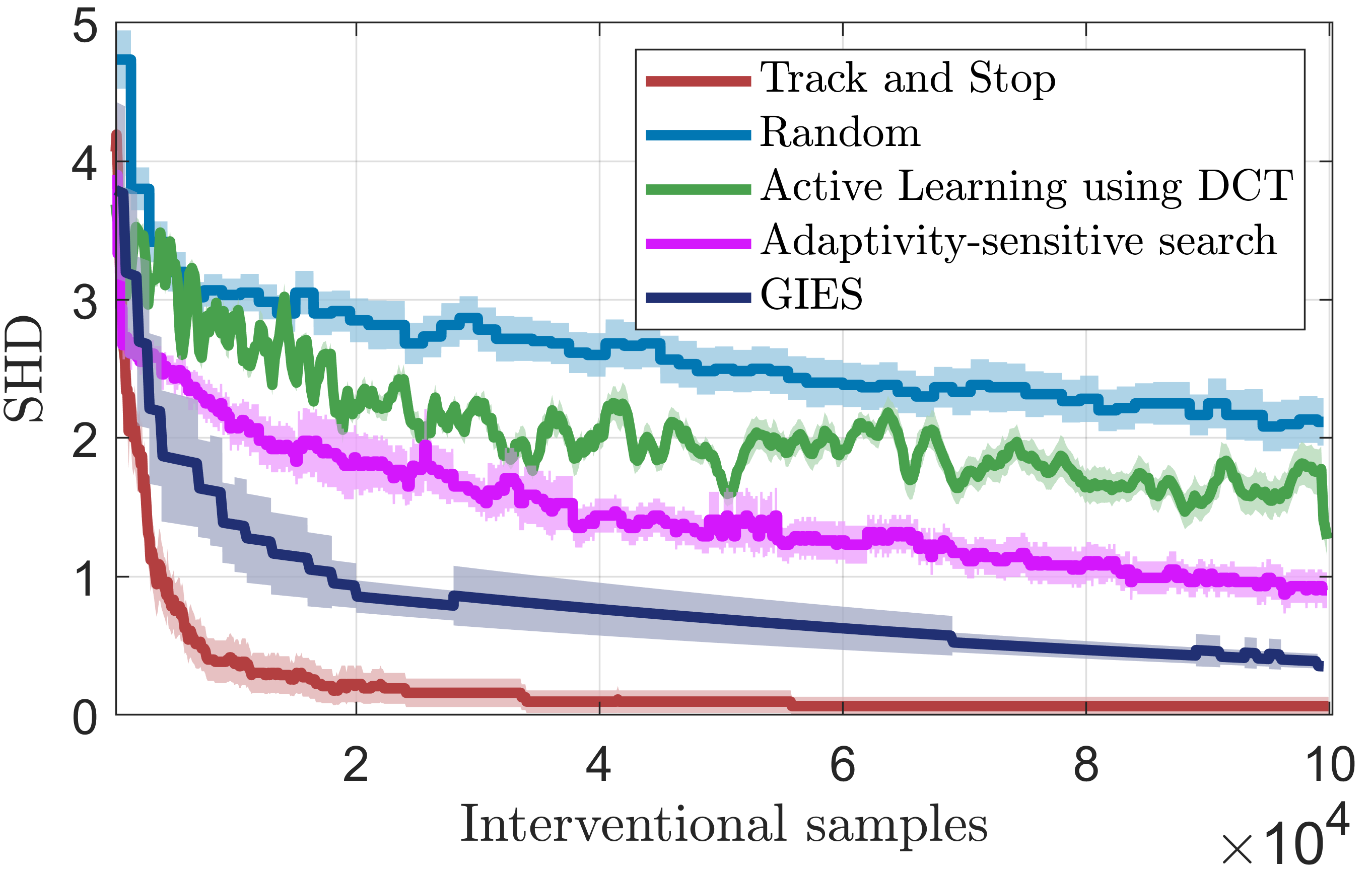}
    \caption{SHD vs No. of samples for SACHS dataset.}
    \label{semisyn}
\end{figure}

Figures \ref{syn_res1} and \ref{syn_res2} plot the Structural Hamming Distance (SHD) between the true and learned DAGs in relation to the number of interventional samples. SHD measures the number of edge additions, deletions, and reversals needed to transform one DAG into another. The shaded region  represents a range of two standard deviations above and below the mean SHD. For statistical independence tests with limited samples, we use the Chi-Square independence test from the Causal Discovery Toolbox \citep{kalainathan2020causal}.
 
The results in Figures \ref{syn_res1} and \ref{syn_res2} show that the track-and-stop algorithm outperforms other causal discovery algorithms. In Figure \ref{syn_res1}, we show the performance of causal discovery algorithms on complete graphs with $5$, $6$, and $7$ vertices, demonstrating the better performance of the track-and-stop algorithm compared to other baseline methods. The number of samples required by the other algorithms to achieve a low SHD increases significantly faster with the number of nodes compared to our proposed algorithm. A comparison of the plots in Figure \ref{syn_res2}(a), \ref{syn_res2}(b), and \ref{syn_res2}(c) reveals that as the DAGs become denser, the number of samples required by our algorithm does not increase significantly. In contrast, other causal discovery algorithms experience a significant increase in SHD in denser graphs compared to sparser ones, requiring a larger number of samples to achieve a low SHD.

We also assess the performance of causal discovery algorithms using the SACHS Bayesian network, consisting of 13 nodes and 17 edges from the bnlearn library \citep{scutari2009learning}. The SACHS dataset measures expression levels of various proteins and phospholipids in human cells \citep{sachs2005causal}. As shown in Figure \ref{semisyn}, the track-and-stop algorithm outperforms other baseline methods, resulting in significantly lower SHD for the same number of samples. The evaluation results from these synthetic and semi-synthetic experiments establish the superior performance of the proposed algorithm. While our setup requires access to the true observational distribution, in the supplementary section, we explore a scenario where our algorithm starts with the wrong CPDAG due to limited observational data, causing SHD to settle at some non-zero value instead of zero.

\section{Conclusion}
Causal discovery aims to reconstruct the causal structure that explains the mechanism of the underlying data-generating process through observation and experimentation. Inspired by pure exploration problems in bandits, we propose a track-and-stop causal discovery algorithm that intervenes adaptively and employs a decision rule to return the most probable causal graph at any stage. We establish a problem-dependent upper bound on the expected number of interventions by the algorithm. We conduct a series of experiments on synthetic and semi-synthetic data and demonstrate that the track-and-stop algorithm outperforms many baseline causal discovery algorithms, requiring considerably fewer interventional samples to learn the true causal graph.

\section{Acknowledgements}
Murat Kocaoglu acknowledges the support of NSF CAREER 2239375, Amazon Research Award, and Adobe Research.

\bibliographystyle{plainnat}
\bibliography{main}  

\newpage
\appendix
\onecolumn
\section{Supplementary Material} 

\subsection{Procedure to construct $(n,k)$ seperating system}

\begin{lemma}[\cite{shanmugam2015learning}]
   There exists a labeling procedure that gives distinct labels of length $\ell$ for all elements in $[n]$ using letters from the integer alphabet $\{0,1 \ldots a \}$, where $\ell=\lceil \log_{a} n \rceil$. Furthermore, in every position, any integer letter is used at most $\lceil n/a \rceil$ times.
   \label{lbl prc}
\end{lemma}

The string labeling method in Lemma \ref{lbl prc} from \cite{shanmugam2015learning} is described below:

\textbf{Labelling Procedure:} Let $a>1$ be a positive integer. Let $x$ be the integer such that $a^{x} < n \leq a^{x+1}$. $x+1 = \lceil \log_{a} n \rceil$. Every element $j \in [1:n]$ is given a label $L(j)$ which is a string of integers of length $x+1$ drawn from the alphabet $\{0,1,2 \ldots a\}$ of size $a+1$. Let $n= p_d a^d+r_d$ and $n=p_{d-1}a^{d-1}+r_{d-1}$ for any integers $p_d,p_{d-1},r_{d},r_{d-1}$, where $r_d < a^d$ and $r_{d-1}<a^{d-1}$. Now, we describe the sequence of the $d$-th digit across the string labels of all elements from $1$ to $n$: 

  \begin{enumerate}
   \item Repeat the integer $0$ a total of $a^{d-1}$ times, and then repeat the subsequent integer, $1$, also $a^{d-1}$ times \footnote{Circular means that after $a-1$ is completed, we start with $0$ again.} from $\{0,1 \ldots a-1 \}$ till $p_da^d$. 
   \item Following this, repeat the integer $0$ a number of times equal to $\lceil r_d/a \rceil$, and then repeat the integer $1$ $\lceil r_d/a \rceil$ times, continuing this pattern until we reach the $n$th position. It is evident that the $n$th integer in the sequence will not exceed $a-1$.
   \item Each integer that appears beyond the position $a^{d-1}p_{d-1}$ is incremented by $1$.
  \end{enumerate}  

Once we have a set of $n$ string labels, we can easily construct a $(n,k)$ separating system using Lemma \ref{const}, stated as follows:

\begin{lemma}[\cite{shanmugam2015learning}]
 Consider an alphabet ${\cal A}=[0:\lceil \frac{n}{k} \rceil]$ of size $\lceil \frac{n}{k} \rceil+1$ where $k<n/2$. Label every element of an $n$ element set using a distinct string of letters from ${\mathcal A}$ of length $ \ell= \lceil \log_{\lceil \frac{n}{k} \rceil} n \rceil$ using the labeling procedure in Lemma \ref{lbl prc} with $a=\lceil \frac{n}{k} \rceil$. For every $1 \leq a \leq \ell$ and $1 \leq b \leq \lceil \frac{n}{k} \rceil $,we choose the subset $I_{a,b}$ of vertices whose string's $a$-th letter is $b$. The set of all such subsets ${\mathcal S} = \{s_{a,b}\}$ is a $k$-separating system on $n$ elements and $\lvert {\mathcal S} \rvert \leq (\lceil \frac{n}{k} \rceil ) \lceil \log_{\lceil \frac{n}{k} \rceil} n \rceil $.
\label{const}
\end{lemma}

\subsection{Meek Rules}
 The following algorithm can be used to apply Meek orientation rules to PDAGs.
\begin{algorithm}
\small
\SetAlgoLined
\DontPrintSemicolon
\SetKwFunction{FMain}{ApplyMeekRules}
\SetKwProg{Fn}{Function}{:}{}
\Fn{\FMain{$\mathcal{M}$}}{
    Orient as many undirected edges as possible by repeated
    application of the following three rules:\\
    \textbf{(R1)} Orient $b - c$ into $b \rightarrow c$ whenever there is an arrow $a \rightarrow b$ such that $a$ and $c$ are nonadjacent.\\
    \textbf{(R2)} Orient $a - b$ into $a \rightarrow b$ whenever there is a chain $a \rightarrow c \rightarrow b$.\\
    \textbf{(R3)} Orient $a - b$ into $a \rightarrow b$ whenever there are two chains $a - k \rightarrow b$ and  $a - l \rightarrow b$ such that $k$\\ and $l$ are nonadjacent.\\
   \textbf{(R4)} Orient $a - b$ into $a \rightarrow b$ whenever there is an edge $a - k$ and chain $k \rightarrow l \rightarrow b$ such that $k$ \\and
    $b$ are nonadjacent.
}
\textbf{return}  A valid MPDAG: $\mathcal{M}$\\
\textbf{End Function}
\caption{Apply Meek Rules to a Skeleton}
\label{MeekRules}
\end{algorithm}

\subsection{Algorithms to find  Partial Causal Ordering (PCO) and Enumerate all possible causal effects in the MPDAG}

\begin{definition}
(Bucket \citep{perkovic2020identifying}) Consider an MPDAG $\mathcal{M}(\mathbf{V},\mathbf{E})$ and set of vertices $\mathbf{S} \in \mathbf{V}$. The maximal undirected connected subset of $S$ in $\mathcal{M}$ is defined as a bucket in  $S$.
\label{def_bucket}
\end{definition}

Definition \ref{def_bucket} permits the presence of directed edges connecting nodes within the same bucket. This definition allows for a unique decomposition, known as the bucket decomposition, to be applied to any set of vertices in the MPDAG.

\begin{algorithm}[t!]
\small
\SetAlgoLined
\DontPrintSemicolon
    \SetKwFunction{FMain}{$\mathsf{PCO}$}
    \SetKwProg{Fn}{Function}{:}{}
    \Fn{\FMain{$\mathcal{M(\mathbf{V},\mathbf{E})},\mathbf{S}$}}{
        $ \mathbf{CC} =$ Bucket decomposition of $\mathbf{V}$ in $\mathcal{M}$\\
        $\mathbf{B} =$ an empty list\\
       \While{$\mathbf{CC} \neq \emptyset$}{
        Let $\mathbf{c} \in \mathbf{CC}$ \qquad//First element in set $CC$\\
        $\overline{\mathbf{c}} = \mathbf{CC} \setminus \mathbf{c}$\\
       \If{all edges in $E(\mathbf{c},\overline{\mathbf{c}})$ have a head in $\mathbf{c}$ } 
         {
            $\mathbf{CC} = \overline{\mathbf{c}}$\\
            $\overline{B} = \mathbf{S} \cap \mathbf{c}$\\             
                \If{$\overline{B} \neq \emptyset$} 
                {
                   Add $\overline{B}$ to the beginning of $\mathbf{B}$\\
                                }
                  }
              }
    
       }
        \textbf{return} $\mathbf{B}$ (An ordered list of Bucket Decomposition of $\mathbf{S}$) 

\textbf{End Function}
\caption{Partial Causal Ordering \citep{perkovic2020identifying}}
\label{PCO}
\end{algorithm}

\begin{lemma}
    (Bucket Decomposition \citep{perkovic2020identifying}) Consider an MPDAG $\mathcal{M}(\mathbf{V},\mathbf{E})$ and set of vertices $\mathbf{S} \in \mathbf{V}$. There exists a unique partition of $\mathbf{S}$ into pairwise disjoint subsets $(\mathbf{B}_1,\mathbf{B}_2,....,\mathbf{B}_m)$ such that $\bigcup_{i=1}^{m} \mathbf{B}_i = \mathbf{S}$ and $\mathbf{B}_i$ is a bucket of $\mathbf{S}\;\forall i \in [m]$.
\end{lemma}
The Algorithm \ref{PCO} returns an ordered list of bucket decomposition of $\mathbf{S}$ in $\mathcal{M}$. Also ordered list of
buckets output by Algorithm \ref{PCO} is a partial causal ordering of $\mathbf{S}$ in $\mathcal{M}$.

\begin{algorithm}[t!]
\small
\SetAlgoLined
\DontPrintSemicolon
\SetKwFunction{FMain}{IdentifyCausalEffect}
\SetKwInOut{Input}{Input}
\SetKwInOut{Set}{Set}
		\SetKwInOut{Title}{Algorithm}
		\SetKwInOut{Require}{Require}
		\SetKwInOut{Output}{Output}
            \SetKwInOut{Initialization}{Initialization}
		
		{	
			\Input{MPDAG $\mathcal{M(\mathbf{V},\mathbf{E})}$ , $\mathbf{X},\mathbf{Y} \subseteq \mathbf{V}$ }

		\Output{The interventional distribution $P(\mathbf{y}|do(\mathbf{x}))$ in MPDAG }
		}

\SetKwProg{Fn}{Function}{:}{}
\Fn{\FMain{$\mathcal{M(\mathbf{V},\mathbf{E})},\mathbf{X},\mathbf{Y}$}}{
    $(\mathbf{B}_1,\mathbf{B}_2,....,\mathbf{B}_m)$ =$\mathsf{PCO}(\mathsf{An}(\mathbf{Y},\mathcal{M_{\mathbf{V} \setminus \mathbf{X} }}),\mathcal{M})$ \\
    $\mathbf{b} = $ $\mathsf{An}(\mathbf{Y},\mathcal{M_{\mathbf{V} \setminus \mathbf{X} }}) \setminus \mathbf{Y}$\\
    \vspace{0.02cm}
    $P(\mathbf{y}|do(\mathbf{x})) =\sum_{\mathbf{b}} \prod_{i=1}^{m} P(\mathbf{b}_i|\mathsf{Pa}(\mathbf{b}_i,\mathcal{M}))$\\
   
}
\textbf{return} $P(\mathbf{y}|do(\mathbf{x}))$\\
\textbf{End Function}
\caption{Identify Causal Effect in an MPDAG}
\label{ident_caus}
\end{algorithm}

\begin{algorithm}[t!]
\small
\SetAlgoLined
\DontPrintSemicolon
\SetKwFunction{FMain}{EnumerateCausalEffect}
\SetKwInOut{Input}{Input}
\SetKwInOut{Set}{Set}
		\SetKwInOut{Title}{Algorithm}
		\SetKwInOut{Require}{Require}
		\SetKwInOut{Output}{Output}
            \SetKwInOut{Initialization}{Initialization}
		
		{	
			\Input{MPDAG $\mathcal{M(\mathbf{V},\mathbf{E})}$ , $\mathbf{X},\mathbf{Y} \subseteq \mathbf{V}$ }

		\Output{All possible interventional distribution $P(\mathbf{y}|do(\mathbf{x}))$ in MPDAG }
		}
\SetKwProg{Fn}{Function}{:}{}
\Fn{\FMain{$\mathcal{M(\textbf{V},\textbf{E})},\textbf{X},\textbf{Y}$}}{
    $\textbf{List} =$ an empty list\\
    $\textbf{E} =$ Unoriented edges in cut at $\textbf{X}$\\
    \uIf{$\textbf{e} = \emptyset$}{
     $P(\textbf{y}|do(\textbf{x}))$ = IdentifyCausalEffect($\mathcal{M},\textbf{X},\textbf{Y}$)\\
     Add $P(\textbf{y}|do(\textbf{x}))$  to the $\textbf{List}$\\
  }
    \Else{
            \For{All possible orientations of edges in $\textbf{E}$}
            { Orient the corresponding edges $\textbf{E}$ in $\mathcal{M}$ to get $\hat{\mathcal{M}}$ \\
            $\bar{\mathcal{M}}$ = ApplyMeekRules( $\hat{\mathcal{M}}$ )\\
            $P(\textbf{y}|do(\textbf{x}))$ = IdentifyCausalEffect($\bar{\mathcal{M}},\textbf{X},\textbf{Y}$)\\
            Add $P(\textbf{y}|do(\textbf{x}))$  to the $\textbf{List}$\\
            
            }

  }
}
\textbf{return} $\textbf{List}$ (A List of all candidate values of $P(\textbf{y}|do(\textbf{x}))$ for all DAGs in $[\mathcal{M}]\;$) \\
\textbf{End Function}
\caption{Enumerate Causal Effect in an MPDAG}
\label{enum_caus}
\end{algorithm}

Algorithm \ref{enum_caus} provides a systematic procedure for enumerating all possible values for $P(\mathbf{y}|\mathbf{x})$ in a given MPDAG.

\subsection{Proof of Lower Bound in~\cref{th: lowerbound}}
The lower bound is derived following the same strategy in~\cite{lattimore2020bandit} by applying divergence decomposition and Bretagnolle–Huber inequality. For completeness, we reproduce both proofs in this section. Readers familiar with these results can skip them.

Recall that a policy $\pi$ is composed of a sequence $\seqdef{\pi_t}{t \in \natural_{>0}}$, where at each time $t \in \until{T}$, $\pi_t$ determines the probability distribution of taking intervention $\mathbf{s}_t \in \mathcal{I}$ given intervention and observation history $\pi_t(\mathbf{s}_t \mid \mathbf{s}_1, \mathbf{v}_1, \ldots , \mathbf{s}_{t-1}, \mathbf{v}_{t-1})$. So the intervention and observation sequence $\seqdef{\mathbf{s}_t, \mathbf{v}_t}{t\in \natural_{>0}}$ is a production of the interactions between the interventional distribution tuple $\tupdef{P_\mathbf{s} }{\mathbf{s} \in \mathcal{I}}$ and policy $\pi$. For any $T \in \natural_{> 0}$, we define a probability measure $\mathbb{P}$ on the sequence of outcomes induced by $\tupdef{P_\mathbf{s}}{\mathbf{s}\in \mathcal{I}}$ and $\pi$ such that
\begin{equation}\label{eq: decompose}
    \mathbb{P} (\mathbf{s}_1, \mathbf{v}_1 , \ldots, \mathbf{s}_T, \mathbf{v}_T ) = \prod_{t=1}^T \pi_t( \mathbf{s}_t \mid \mathbf{s}_1, \mathbf{v}_1 , \ldots, \mathbf{s}_{t-1}, \mathbf{v}_{t-1} ) P_{\mathbf{s}_t}(\mathbf{v}_t).
\end{equation}
The following decomposition is a standard result in Bandit literature~\citep[Ch. 15]{lattimore2020bandit}.

\begin{lemma}[Divergence Decomposition] \label{lemma: DD}
In the causal discovery problem, assume $\mathcal{D}^*$ is the true DAG. for any fixed policy $\pi$, let $\mathbb{P}$ and $\mathbb{P}'$ be the probability measures corresponding to applying interventions on $\mathcal{D}^*$ and $\mathcal{D}'$, respectively. Let $\mathcal{F}=\seqdef{\mathcal{F}_t}{t\in \natural_{>0}}$ be a filtration, where $\mathcal{F}_t = \sigma(\mathbf{s}_1, \mathbf{v}_1, \ldots, \mathbf{s}_{t-1}, \mathbf{v}_{t-1})$, and let $\tau$ be a $\mathcal{F}$-measurable stopping time. Then for any event $E$ that is $\mathcal{F}_{\tau}$ measurable,
\[ \kl{\prob(E)}{\prob'(E)} = \sum_{\mathbf{s} \in \mathcal{I}} \expt [N_{\tau} (\mathbf{s}) ] \kl{P^{\mathcal{D}*}_{\mathbf{s}}}{P^{\mathcal{D}}_{\mathbf{s}}},\]
where the expectation is computed with probability measure $\prob$.
\end{lemma}

\begin{proof}
For a given sequence $\seqdef{\mathbf{s}_t, \mathbf{v}_t}{t\in \natural_{>0}}$, let $\tau$ be the stopping time. Since policy $\pi$ and stopping time $\tau$ are fixed, it follows from~\eqref{eq: decompose} that 
\[\mathbb{P} (\mathbf{s}_1, \mathbf{v}_1 , \ldots, \mathbf{s}_{\tau}, \mathbf{v}_{\tau} ) = \prod_{t=1}^{\tau} \pi_t( \mathbf{s}_t \mid \mathbf{s}_1, \mathbf{v}_1 , \ldots, \mathbf{s}_{t-1}, \mathbf{v}_{t-1} ) P_{\mathbf{s}_t}(\mathbf{v}_t).\]
 Accordingly, we define random variable
\begin{equation}\label{eq: ratio_P_P'}
     L_{\tau} := \log \frac{\prob (\mathbf{s}_1, \mathbf{v}_1 , \ldots, \mathbf{s}_{\tau}, \mathbf{v}_{\tau} )}{\prob' (\mathbf{s}_1, \mathbf{v}_t , \ldots, \mathbf{s}_{\tau}, \mathbf{v}_{\tau} )} = \sum_{t=1}^{\tau} \log \frac{P^{\mathcal{D}*}_{\mathbf{s}_t}(\mathbf{v}_t)}{P^{\mathcal{D}'}_{\mathbf{s}_t}(\mathbf{v}_t)} ,
 \end{equation}
in which $\pi_t$ is reduced. Equation~\eqref{eq: ratio_P_P'} shows that the distinction between $\mathbb{P}$ and $\mathbb{P}'$ is exclusively due to the separations of $P_\mathbf{s}$ and $P'_\mathbf{s}$ for each $\mathbf{s} \in \mathcal{I}$. Let $\seqdef{\mathbf{v}_{\mathbf{s},i}}{i \in \natural_{>0}}$ be the sequence of observations by applying intervention $do(\mathbf{S=s} )$. Then we have
\begin{equation}\label{eq: wald}
    L_{\tau} = \sum_{\mathbf{s} \in \mathcal{I}} \sum_{i=1}^{N_{\tau}(\mathbf{s})} \log \frac{P_{\mathbf{s}}(\mathbf{v}_{\mathbf{s},i})}{P'_{\mathbf{s}}(\mathbf{v}_{\mathbf{s},i})} \text{ and } \expt \bigg[\log \frac{P_{\mathbf{s}}(\mathbf{v}_{\mathbf{s},i})}{P'_{\mathbf{s}}(\mathbf{v}_{\mathbf{s},i})} \bigg] = \kl{P^{\mathcal{D}*}_{\mathbf{s}}}{P^{\mathcal{D}}_{\mathbf{s}}}.
\end{equation}
Since event $E$ that is $\mathcal{F}_{\tau}$ measurable, we apply log sum inequality to get
\[\kl{\mathbb{P}(E)}{\mathbb{P}'(E)} \leq \expt \left[ \log \frac{\mathbb{P} (\mathbf{s}_1, \mathbf{v}_1 , \ldots, \mathbf{s}_{\tau} , \mathbf{v}_{\tau}  )}{\mathbb{P}' (\mathbf{s}_1, \mathbf{v}_t , \ldots, \mathbf{s}_{\tau} , \mathbf{v}_{\tau} )} \right] = \expt[L_\tau].\]
With~\eqref{eq: wald}, we apply Wald's Lemma (e.g.~\cite{siegmund1985sequential}) to get
\[\kl{\mathbb{P}(E)}{\mathbb{P}'(E)} \leq \expt \left[ \sum_{\mathbf{s} \in \mathcal{I}} \sum_{i=1}^{N_{\tau}(\mathbf{s})} \log \frac{P_{\mathbf{s}}(\mathbf{v}_{\mathbf{s},i})}{P'_{\mathbf{s}}(\mathbf{v}_{\mathbf{s},i})} \right] = \sum_{\mathbf{s} \in \mathcal{I}} \expt [N_{\tau} (\mathbf{s}) ] \kl{P^{\mathcal{D}*}_{\mathbf{s}}}{P^{\mathcal{D}}_{\mathbf{s}}}, \]
which concludes the proof.
\end{proof}

The other tool to prove the regret lower bound is the Bretagnolle–Huber Inequality.
\begin{lemma}[{Bretagnolle–Huber Inequality~\cite{lattimore2020bandit}, Th 14.2}] \label{lemma: BH_ineq}
    Let $P$ and $Q$ be two probability measures on a measurable space $(\omega, \mathcal{F})$, and let $E \in \mathcal{F}$ be an arbitrary event. Then
    \[P(E) + Q(E^\complement) \geq \frac{1}{2} \exp \left(- \kl{P}{Q} \right ),\]
    where $E^\complement = \omega \setminus E$ is complement of $E$.
\end{lemma}

\begin{proof}[Proof of~\cref{th: lowerbound}]
    If $\expt[\tau_{\delta}] = \infty$, the result is trivial. Assume that $\expt[\tau_{\delta}] < \infty$, which also indicates $\prob(\tau_{\delta} = \infty) = 0$. Recall $\mathbb{P}$ and $\mathbb{P}'$ are the probability measures corresponding to applying interventions on $\mathcal{D}^*$ and $\mathcal{D}'$ respectively. Define event $E = \{ \tau_{\delta} \leq \infty, \psi \neq \mathcal{D}' \}$. For a sound casual discovery algorithm, we have
    \begin{align}
        2\delta & \geq \prob \big( \tau_{\delta} \leq \infty, \psi \neq \mathcal{D}^* \big ) + \prob' \big( \tau_{\delta} \leq \infty, \psi \neq \mathcal{D}' \big ) \\
        &\geq \prob \big( E^\complement \big ) + \prob' ( E ).         
    \end{align}
    We apply Bretagnolle–Huber inequality to get
    \begin{equation} \label{ineq: BH inequality}
        2\delta \geq \frac{1}{2} \exp \left(- \kl{\prob(E)}{\prob'(E)} \right ).
    \end{equation}
     With~\cref{lemma: DD}, we substitute $\sum_{\mathbf{s}\in \mathcal{I}}\expt [N_T (\mathbf{s}) ] \kl{P^{\mathcal{D}*}_{\mathbf{s}}}{P^{\mathcal{D}'}_{\mathbf{s}}}$ into $ \kl{\prob(E)}{\prob'(E)}$  and rearrange~\eqref{ineq: BH inequality} to get 
     \begin{equation}\label{ineq: bh_rst}
         \log \frac{4}{\delta} \leq \sum_{\mathbf{s}\in \mathcal{I}}\expt [N_{\tau_{\delta}} (\mathbf{s}) ] \kl{P^{\mathcal{D}*}_{\mathbf{s}}}{P^{\mathcal{D}'}_{\mathbf{s}}} \leq \expt[\tau_{\delta}] \sum_{\mathbf{s}\in \mathcal{I}} \frac{\expt [N_{\tau_{\delta}}(\mathbf{s})]}{\expt[\tau_{\delta}]}  \kl{P^{\mathcal{D}*}_{\mathbf{s}}}{P^{\mathcal{D}'}_{\mathbf{s}}} \leq \expt[\tau_{\delta}] c(\mathcal{D}^*).
     \end{equation} 
     Since~\eqref{ineq: bh_rst} holds for any $\mathcal{D}' \in [\mathcal{C}]\setminus \mathcal{D}^*$, we have
     \begin{align*}
         \log \frac{4}{\delta} &\leq \expt[\tau_{\delta}] \min_{\mathcal{D} \in [\mathcal{C}]\setminus \mathcal{D}^*} \sum_{\mathbf{s}\in \mathcal{I}} \frac{\expt [N_{\tau_{\delta}}(\mathbf{s})]}{\expt[\tau_{\delta}]}  \kl{P^{\mathcal{D}^*}_{\mathbf{s}}}{P^{\mathcal{D}}_{\mathbf{s}}} \\
         &\leq \expt[\tau_{\delta}] \max_{\boldsymbol{\alpha} \in \Delta(\mathcal{I})} \min_{ \mathcal{D} \in [\mathcal{C}] \setminus \mathcal{D}^* } \sum_{\mathbf{s} \in \mathcal{I}}  \alpha_{\mathbf{s}} \KL{P^{\mathcal{D}^*}_{\mathbf{s}}}{P^{\mathcal{D}}_{\mathbf{s}}} =  \expt[\tau_{\delta}] c(\mathcal{D}^*),
     \end{align*}
     where the last inequality is due to $\sum_{\mathbf{s}\in \mathcal{I}} {\expt [N_{\tau_{\delta}}(\mathbf{s})]}/{\expt[\tau_{\delta}]} = 1$. We conclude the proof.
\end{proof}

\subsection{Supporting Lemmas for~\cref{th: upperbound}}
\subsubsection{Supporting Lemmas on Online Maxmin Optimization}
In~\eqref{eq: lowerbound}, we define a variable $c(\mathcal{D}^*) = \sup_{\boldsymbol{\alpha} \in \Delta(\mathcal{I})} \min_{ \mathcal{D} \in [\mathcal{C}] \setminus \mathcal{D}^* } \sum_{\mathbf{s} \in \mathcal{I}}  \alpha_{\mathbf{s}} \KL{P^{\mathcal{D}^*}_{\mathbf{s}}}{P^{\mathcal{D}}_{\mathbf{s}}}$. We define another variable for the local discovery result
\begin{equation}\label{eqdef: local}
    c_{\mathbf{S}}(\mathcal{D}^*) = \sup_{\boldsymbol{\xi}^{\mathbf{S}} \in \Delta(\omega(\mathbf{S}))} \min_{\mathsf{C}(\mathbf{S}) \neq \mathsf{C}^*(\mathbf{S})} \sum_{\mathbf{s} \in \omega(\mathbf{S})} \xi^{\mathbf{S}}_{\mathbf{s}} \kl{P^{\mathsf{C}^*(\mathbf{S})}_{\mathbf{s}}}{P^{\mathsf{C}(\mathbf{S})}_{\mathbf{s}, t}}.
\end{equation}
Let $\mathsf{CF}(\mathbf{S})$ denote all the possible configurations of cutting edges attached to the node set $\mathbf{S}$. The following theorem shows these values from the maxmin optimization equal to their minmax counterparts.

\begin{lemma}~\label{lemma: duality}
The following two inequality holds:
    \[c(\mathcal{D}^*) = \inf_{\boldsymbol{w}\in \Delta([\mathcal{C}]\setminus \mathcal{D}^*)} \max_{\mathbf{s} \in \mathcal{I}}  \sum_{\mathcal{D} \in [\mathcal{C}] \setminus \mathcal{D}^*} w_{\mathcal{D}} \KL{P^{\mathcal{D}^*}_{\mathbf{s}}}{P^{\mathcal{D}}_{\mathbf{s}}},\]
    \[c_{\mathbf{S}}(\mathcal{D}^*) = \inf_{\boldsymbol{\zeta}^{\mathbf{S}} \in \Delta(\mathsf{CF}(\mathbf{S})\setminus \mathsf{C}^*(\mathbf{S}))} \max_{\mathbf{s} \in \omega(\mathbf{S})} \sum_{\mathsf{C}(\mathbf{S}) \in \mathsf{CF}(\mathbf{S}) \setminus \mathsf{C}^*(\mathbf{S})} \zeta^{\mathbf{S}}_{\mathsf{C}(\mathbf{S})} \kl{P^{\mathsf{C}^*(\mathbf{S})}_{\mathbf{s}}}{P^{\mathsf{C}(\mathbf{S})}_{\mathbf{s}, t}}, \forall \mathbf{S} \in \mathcal{S}.\] 
    Besides,
    \[\underline{c}(\mathcal{D}^*) = \sup_{\boldsymbol{\alpha} \in \Delta(\mathcal{I})} \min_{\mathbf{S} \in \mathcal{S}} \min_{\mathsf{C}(\mathbf{S}) \neq \mathsf{C}^*(\mathbf{S})} \sum_{\mathbf{s} \in \omega(\mathbf{S})} \alpha_{\mathbf{s}} \kl{P^{\mathcal{D}^*\!\!}_{\mathbf{s}}}{P^{\mathsf{C}(\mathbf{S})}_{\mathbf{s}}} = \gamma_{\mathbf{S}}^* c_{\mathbf{S}}(\mathcal{D}^*), \text{ where } \gamma_{\mathbf{S}}^* = \frac{1 / c_{\mathbf{S}}(\mathcal{D}^*)}{\sum_{\mathbf{S} \in \mathcal{S}} 1 / c_{\mathbf{S}}(\mathcal{D}^*) }.\]
\end{lemma}
\begin{proof}[Sketch of Proof]
    These two max-min optimization problems correspond to designing a mixed strategy for matrix games. To elaborate, the reward matrix $R \in \mathbb{R}^{(\lvert[\mathcal{C}]\rvert-1) \times \lvert\mathcal{I}\rvert}$ has entries represented as $\KL{P^{\mathcal{D}^*}{\mathbf{s}}}{P^{\mathcal{D}}{\mathbf{s}}}$, and solving for $c(\mathcal{D}^)$ is equivalent to the following optimization problem:
\begin{align*}
    &\maximize \min_{i= \until{\abs{[\mathcal{C}]}-1}} (R \boldsymbol{\alpha})_i \\
    &\subject \quad \alpha  \succeq 1, \mathbf{1}^T \boldsymbol{\alpha} = 1.
\end{align*}
For such a problem, it is shown in~\citep[CH 5.2.5]{boyd2004convex} that strong duality holds. Similar argument can be made on $c_{\mathbf{S}}(\mathcal{D}^*)$. Detailed proofs are omitted.

To prove the last equality, let $\gamma_{\mathbf{S}} = \sum_{\mathbf{s}\in \omega(\mathbf{S})} \alpha_{\mathbf{s}}$ and let $ \alpha_{\mathbf{s}} = \gamma_{\mathbf{S}} \xi^{\mathbf{S}}_{\mathbf{s}}$. We also have $\sum_{\mathbf{S} \in \mathcal{S} } \gamma_{\mathbf{S}} = 1$ and $\sum_{\mathbf{s}\in \omega(\mathbf{S})} \xi^{\mathbf{S}}_{\mathbf{s}} = 1$. It follows that,
\begin{align*}
    \underline{c}(\mathcal{D}^*) &= \sup_{\boldsymbol{\alpha} \in \Delta(\mathcal{I})} \min_{\mathbf{S} \in \mathcal{S}} \min_{\mathsf{C}(\mathbf{S}) \neq \mathsf{C}^*(\mathbf{S})} \sum_{\mathbf{s} \in \omega(\mathbf{S})} \alpha_{\mathbf{s}} \kl{P^{\mathcal{D}^*}_{\mathbf{s}}}{P^{\mathsf{C}(\mathbf{S})}_{\mathbf{s}}} \\
    &= \sup_{\seqdef{\gamma_{\mathbf{S}}}{\mathbf{S} \in \mathcal{S}}}\sup_{\seqdef{\xi^{\mathbf{S}}_{\mathbf{s}}}{\mathbf{s}\in \omega(\mathbf{S})}} \min_{\mathbf{S} \in \mathcal{S}} \min_{\mathsf{C}(\mathbf{S}) \neq \mathsf{C}^*(\mathbf{S})} \sum_{\mathbf{s} \in \omega(\mathbf{S})} \gamma_{\mathbf{S}} \xi^{\mathbf{S}}_{\mathbf{s}} \kl{P^{\mathcal{D}^*}_{\mathbf{s}}}{P^{\mathsf{C}(\mathbf{S})}_{\mathbf{s}}}\\
    &= \sup_{\seqdef{\gamma_{\mathbf{S}}}{\mathbf{S} \in \mathcal{S}}} \gamma_{\mathbf{S}} \min_{\mathbf{S} \in \mathcal{S}} \sup_{\seqdef{\xi^{\mathbf{S}}_{\mathbf{s}}}{\mathbf{s}\in \omega(\mathbf{S})}} \min_{\mathsf{C}(\mathbf{S}) \neq \mathsf{C}^*(\mathbf{S})} \sum_{\mathbf{s} \in \omega(\mathbf{S})}  \xi^{\mathbf{S}}_{\mathbf{s}} \kl{P^{\mathcal{D}^*}_{\mathbf{s}}}{P^{\mathsf{C}(\mathbf{S})}_{\mathbf{s}}}\\
    &= \sup_{\seqdef{\gamma_{\mathbf{S}}}{\mathbf{S} \in \mathcal{S}}} \gamma_{\mathbf{S}} \min_{\mathbf{S} \in \mathcal{S}} c_{\mathbf{S}}(\mathcal{D}^*).
\end{align*}
Besides, the solution for above problem satisfies $\gamma_{\mathbf{S}} \propto 1/c_{\mathbf{S}}(\mathcal{D}^*)$. We conclude the proof.
\end{proof}

\subsubsection{Supporting Lemma for AdaHedge Algorithm}
The AdaHedge deals with such a sequential decision-making problem. At each $t = 1,2, \ldots$, the learner needs to decide a weight vector $\boldsymbol{\alpha}_t = (\alpha_{1,t}, \ldots, \alpha_{K,t})$ over $K$ “experts”. Nature then reveals a $K$-dimensional vector containing the rewards of the experts $\boldsymbol{r}_t = (r_{1,t}, \ldots,r_{K,t}) \in \real^K$. The actual received reward is the dot product $h_t = \boldsymbol{\alpha}_t \cdot  \boldsymbol{r}_t$, which can be interpreted as the expected loss with a mixed strategy. The learn's task it to maximize the cumulative reward $H_T = \sum_{t=1}^T h_t$ or equivalently minimize the regret defined as
\[R_T = \max_{k \in \until{K}} \sum_{t=1}^T r_{k,t} - H_T.\]
The performance guarantee of AdaHedge is as follows.
\begin{lemma}[~\citep{de2014follow}]\label{lemma: adh}
    If for any $t \in \natural_{>0}$, $r_{k,t} \in [0, D]$ for all $k\in\until{K}$, let $\supscr{R}{AH}_T $ be the regret for AdaHedge for horizon $T$. It satisfies that    
    \[\supscr{R}{AH}_T \leq \sqrt{D T \ln K} + D\left(\frac{4}{3} \ln K + 2\right).\]
\end{lemma}

The following lemma is also used in the proof of~\cref{th: upperbound}.
\begin{lemma} \label{lemma: regret}
    If for any $t \in \natural_{>0}$, $r_{k,t} \in [0, D]$ for all $k\in\until{K}$, for any $T\geq \tau>0$,
    \[\max_{\mathbf{S} \in \mathcal{I}}  \sum_{t= \tau+1 }^T r_{\mathbf{s},t} - \sum_{t=\tau + 1}^{T} h_t \geq R_T - \tau D.\]
\end{lemma}
\begin{proof}
We apply the fact that $max_{k \in \until{K}} \sum_{t=1}^T r_{k,t} \leq \tau D + \max_{k \in \until{K}}  \sum_{t= \tau+1 }^T r_{\mathbf{s},t} $.
\begin{align*}
     R_T &= \max_{k \in \until{K}} \sum_{t=1}^T r_{k,t} - \sum_{t=1}^{T} h_t \leq \tau D + \max_{k \in \until{K}}  \sum_{t= \tau+1 }^T r_{\mathbf{s},t} - \sum_{t=1}^{T} h_t \\
    & \leq \tau D + \max_{\mathbf{S} \in \mathcal{I}}  \sum_{t= \tau+1 }^T r_{\mathbf{s},t} - \sum_{t=\tau + 1}^{T} h_t,
\end{align*}
which concludes the proof.
\end{proof}

In the exact version of track-and-stop causal discovery algorithm $\mathcal{A}_{\mathsf{I}}$, the AdaHege is run with $\abs{\mathcal{I}}$ dimensional reward vector $\tupdef{r_{\mathbf{s},t}}{{\mathbf{s} \in \mathcal{I}}}$ with entries $r_{\mathbf{s},t} = \kl{\Bar{P}_{\mathbf{s}, t}}{P^{\mathcal{D}'_t}_{\mathbf{s}}}$, where
\[\mathcal{D}'_t \in \argmin_{\mathcal{D} \in [\mathcal{M}]\setminus \mathcal{D}^*_t} \sum_{\mathbf{s} \in \mathcal{I}} \alpha_{\mathbf{s}, t} \kl{\Bar{P}_{\mathbf{s}, t}}{P^{\mathcal{D}}_{\mathbf{s}}}.\]

In the practical algorithm $\mathcal{A}_{\mathsf{P}}$, for each $\mathbf{S} \in \mathcal{S}$, the AdaHege is run to compute $\boldsymbol{\xi}^\mathbf{S}_{t}$. The feedback is $\omega(\mathbf{S})$ dimensional vector $\tupdef{r^{\mathbf{S}}_{\mathbf{s},t}}{{\mathbf{s} \in \omega(\mathbf{S})}}$ with entries $r^{\mathbf{S}}_{\mathbf{s},t} = \kl{\Bar{P}_{\mathbf{s}, t}}{P^{\mathsf{C}'_t(\mathbf{S}) }_{\mathbf{s}}}$, where 
    \[\mathsf{C}'_t(\mathbf{S}) \in \argmin_{\mathsf{C}(\mathbf{S}) \neq \mathsf{C}_t^*(\mathbf{S})} \sum_{\mathbf{s} \in \omega(\mathbf{S})} \xi^{\mathbf{S}}_{\mathbf{s}, t} \kl{\Bar{P}_{\mathbf{s}, t}}{P^{\mathsf{C}(\mathbf{S})}_{\mathbf{s}}}.\]
Also in our setup $D = \max_{\mathcal{D} \in [\mathcal{C}]} \sup_{P_{\mathbf{s}}} \kl{P_{\mathbf{s}}}{P_{\mathbf{s}}^{\mathcal{D}}}$,
where KL-divergence follows the convention that $0 \log 0 = 0$ $\log 0/0 = 0$ and $x \log x/0 = +\infty$ for $x > 0$.
\subsubsection{Supporting Lemmas on Allocation Matching}
\begin{lemma}\label{lemma: track}
For the track-and-stop causal discovery algorithm, for any $t \geq \abs{\mathcal{I}}$ and any $\mathbf{s}\in \mathcal{I}$
    \[\sum_{i=1}^t \alpha_{\mathbf{s},i} - (\abs{\mathcal{I}} - 1) (\sqrt{t} + 2) \leq N_t(\mathbf{s}) \leq \max \Big\{1 + \sum_{i=1}^t \alpha_{\mathbf{s},i}, \sqrt{t} + 1 \Big\}.\]
\end{lemma}
\begin{proof}
    We first show that for any $t \geq \mathcal{I}$, the following is true.
    \begin{equation} \label{trpr_eq 1}
        N_t(\mathbf{s}) \leq \max \Big\{1 + \sum_{i=1}^t \alpha_{\mathbf{s},i}, \sqrt{t} + 1 \Big\}.
    \end{equation}
    We prove this claim by induction. At time $t' = \mathcal{I}$, $N_{t'}(\mathbf{s}) = 1$ for all $\mathbf{s}\in \mathcal{I}$, so that~\eqref{trpr_eq 1} is true. Suppose $N_{t'}(\mathbf{s}) \leq \max \Big\{1 + \sum_{i=1}^{t'} \alpha_{\mathbf{s},i}, \sqrt{t} + 1 \Big\}$ is true. If $do(\mathbf{s})$ is not selected at $t'+1$, we have
    \begin{equation}\label{trpr_eq 2}
        N_{t'+1}(\mathbf{s}) = N_{t'}(\mathbf{s})  \leq \max \Big\{1 + \sum_{i=1}^{t'+1} \alpha_{\mathbf{s},i}, \sqrt{t'+1} + 1 \Big\}.
    \end{equation}
    If $do(\mathbf{s})$ is selected at $t'+1$ by force exploration, we have
    \begin{equation}\label{trpr_eq 3}
        N_{t'+1}(\mathbf{s}) = N_{t'}(\mathbf{s}) + 1 < \sqrt{t+1} + 1.
    \end{equation}
    If $do(\mathbf{s})$ is selected at $t'+1$ by allocation matching, since $\sum_{\mathbf{s}\in \mathcal{I}} \sum_{i=1}^{t} \alpha_{\mathbf{s},i} = t$ and $\sum_{\mathbf{s}\in \mathcal{I} } N_{t}(\mathbf{s}) = t$, we have
 \[ \min_{\mathbf{s} \in \mathcal{I}} \frac{N_{t}(\mathbf{s})}{\sum_{i=1}^{t} \alpha_{\mathbf{s}, i}} \leq 1. \]
Accordingly, 
\begin{equation}\label{trpr_eq 4}
    \frac{N_{t+1}(\mathbf{s}_t)}{\sum_{i=1}^{t+1} \alpha_{\mathbf{s}, i}} = \frac{N_{t}(\mathbf{s}_t)}{\sum_{i=1}^{t+1} \alpha_{\mathbf{s}, i}} + \frac{1}{\sum_{i=1}^{t+1} \alpha_{\mathbf{s}, i}}  \leq 1 + \frac{1}{\sum_{i=1}^{t} \alpha_{\mathbf{s}, i}}.
\end{equation}
Combining~\eqref{trpr_eq 2}~\eqref{trpr_eq 3}~\eqref{trpr_eq 4}, we show~\eqref{trpr_eq 1} is true. Also notice that for all $\mathbf{s}\in \mathcal{I}$,
\[ N_t(\mathbf{s}) \leq \max \Big\{1 + \sum_{i=1}^t \alpha_{\mathbf{s},i}, \sqrt{t} + 1 \Big\} \leq \sqrt{t} + 2 +  \sum_{i=1}^t \alpha_{\mathbf{s},i}.\]
It follows from that $\sum_{\mathbf{s}\in \mathcal{I}} \sum_{i=1}^{t} \alpha_{\mathbf{s},i} = t$ and $\sum_{\mathbf{s}\in \mathcal{I} } N_{t}(\mathbf{s}) = t$,
\[N_t(\mathbf{s}) \geq \sum_{i=1}^t \alpha_{\mathbf{s},i} - (\abs{\mathcal{I}} - 1) (\sqrt{t} + 2).\]
We conclude the proof.
\end{proof}

\subsubsection{Supporting Lemmas on Concentration Inequality of Empirical Mean}
 The following \cref{lemma: hoeffding} proposed in~\cite{combes2014unimodal} extends Hoeffding’s inequality to provide an upper bound on the deviation of the empirical mean sampled at a stopping time. In our problem, each time the intervention is selected is a stopping time.
\begin{lemma}[Extension of Hoeffding’s Inequality~\cite{combes2014unimodal}, Lemma 4.3]~\label{lemma: hoeffding}
    Let $\seqdef{Z_t}{t\in \natural_{>0}}$ be  a  sequence  of  independent random  variables  with  values  in $[0, 1]$. Let $\mathcal{F}_t$ be the $\sigma$-algebra such that $\sigma(Z_1, \ldots, Z_t) \subset \mathcal{F}_t$ and  the  filtration $\mathcal{F} = \seqdef{\mathcal{F}_t}{t\in \natural_{>0}}$. Consider $s\in \natural$, and $T \in \natural_{>0}$. We define $S_t = \sum_{j=1}^t \epsilon_j (Z_j - \expt[Z_j]) $, where $\epsilon_j \in \{0,1\}$ is  a $\mathcal{F}_{j-1}$-measurable  random  variable. Further  define $N_t =\sum_{j = 1}^t \epsilon_j$. Define $\phi \in \until{T+1}$ a $\mathcal{F}$-stopping time such that either $N_\phi \geq s$ or $\phi = T+ 1$. Then we have  that
    \[P[S_{\phi} \geq N_{\phi} \delta] \leq \exp(-2s \delta^2).\]
    As a consequence,
    \[P[\abs{S_{\phi}} \geq N_{\phi} \delta] \leq 2 \exp(-2s \delta^2).\]
\end{lemma}

In~\cref{cor: L1_dev}, we extend~\cref{lemma: hoeffding} to bound the $L_1$ deviation of the empirical distribution.
\begin{corollary}[$L_1$ deviation of the empirical distribution] \label{cor: L1_dev}
    Let $\mathcal{A}$ denote finite set $\until{a}$. For two probability distribution $Q$ and $Q'$ on $\mathcal{A}$, let $\norm{Q'-Q}_1 = \sum_{k=1}^a \abs{Q'(k) - Q(k)}$. Let $X_t \in \mathcal{A}$ be a sequence of independent random variables with common distribution $Q$.  Let $\mathcal{F}_t$ be the $\sigma$-algebra such that $\sigma(X_1, \ldots, X_t) \subset \mathcal{F}_t$ and  the  filtration $\mathcal{F} = \seqdef{\mathcal{F}_t}{t\in \natural_{>0}}$. Let $\epsilon_t \in \{0,1\}$ be a $\mathcal{F}_{t-1}$-measurable  random  variable. We define
    \[N_t = \sum_{j=1}^t \epsilon_j, S_t(i) = \sum_{j=1}^t \epsilon_j \indicator{X_j = i} , \text{ and } \Bar{Q}_t(i) = \frac{S_t(i)}{N_t}, \forall i \in \mathcal{A}. \]
    For $s\in \natural$, and $T \in \natural_{>0}$, let $\phi \in \until{T+1}$ be a $\mathcal{F}$-stopping time such that either $N_\phi \geq s$ or $\phi = T+ 1$. Then we have
    \[ P \left(\norm{\Bar{Q}_{\phi} - Q}_1  \geq \delta \right) \leq (2^a - 2) \exp\Big(\frac{-s \delta^2}{2}\Big) .\]
\end{corollary}
\begin{proof}
    It is known that for any distribution $Q'$ on $\mathcal{A}$,
    \[\norm{Q'-Q}_1 = 2 \max_{A \subseteq \mathcal{A}} (Q'(A) - Q(A)).\]
    Then we apply a union bound to get
    \begin{align*}
        P \left(\norm{\Bar{Q}_{\phi} - Q}_1  \geq \delta \right) &\leq \sum_{A \subseteq \mathcal{A}} P \left(\Bar{Q}_{\phi}(A) - Q(A) \geq \frac{\delta}{2} \right) \\
        &\leq \sum_{A \subseteq \mathcal{A}: A \neq \mathcal{A} \text{ or } \emptyset} P \left(\Bar{Q}_{\phi}(A) - Q(A) \geq \frac{\delta}{2} \right) \\
        &\leq (2^a - 2) \exp\Big(\frac{-s \delta^2}{2}\Big),
    \end{align*}
    which concludes the proof.
\end{proof}
\begin{corollary}\label{corr: boundedexpt}
    For the causal discovery problem with the track-and-stop algorithm and any $\epsilon > 0$, define the random time
    \begin{equation*}
        \tau_{p}(\epsilon) = \max \Bigsetdef{t \in \natural_{>0}}{ \exists \mathbf{s} \in \mathcal{I}: \norm{\Bar{P}_{\mathbf{s}, t} - P^{\mathcal{D}^*}_{\mathbf{s}} }_1 > \epsilon}.
    \end{equation*}
    Then there exists a constant $c(\epsilon)>0$ such that $\expt[\tau_{p}(\epsilon)] \leq c(\epsilon)$.
\end{corollary}
\begin{proof}
The forced exploration step guarantees that each intervention is selected at least $\Omega(\sqrt{t})$ times at time $t$. To show that, we first note that the following two facts are true:
    \begin{itemize}
        \item $\min_{\mathbf{s} \in \mathcal{I}} N_t(\mathbf{s})$ is non-decreasing over $t$.
        \item If $\min_{\mathbf{s} \in \mathcal{I}} N_{t_i}(\mathbf{s}) < \sqrt{i} $, then $\min_{\mathbf{s} \in \mathcal{I}} N_{t_{i + \abs{\mathcal{I}}}}(\mathbf{s}) \geq \min N_{t_i}(\mathbf{s}) + 1$.
    \end{itemize}
    Since $N_t(\mathbf{s})$ for each $\mathbf{s} \in \mathcal{I}$ is non-decreasing over $t$, the first statement is true. The second statement is true since otherwise, after at least $\abs{\mathcal{I}}$ forced exploration steps, $\min_{\mathbf{s} \in \mathcal{I}} N_t(\mathbf{s})$ does not increase. With these two facts, we are ready to show for any $\alpha \in (0,1)$ and $t\geq \alpha \abs{\mathcal{I}}^2 /(1-\alpha)^2$, $\min_{\mathbf{s} \in \mathcal{I}} N_{t}(\mathbf{s}) \geq \sqrt{\alpha t}$. The proof is provided by contradiction. Suppose there exists time step $i$ such that 
    \[\min_{\mathbf{s} \in \mathcal{I}} N_{i}(\mathbf{s}) < \sqrt{\alpha i}.\]
    According to the first fact, we have for any $j \geq {\alpha i}$
    \[\min_{\mathbf{s} \in \mathcal{I}} N_{j}(\mathbf{s}) \leq \min_{\mathbf{s} \in \mathcal{I}} N_{i}(\mathbf{s}) < \sqrt{\alpha i}.\]
    Then we apply the second fact. For any $i \geq \alpha \abs{\mathcal{I}}^2 /(1-\alpha)^2$, we have
    \[\min_{\mathbf{s} \in \mathcal{I}} N_{i}(\mathbf{s})\geq \frac{i-j}{\abs{\mathcal{I}}} \geq \frac{(1-\alpha) i} {\abs{\mathcal{I}}} \geq \sqrt{\alpha i},\]
    which creates a contradiction.
    
    To show $\expt[\tau_{p}(\epsilon)] \leq c$, we first notice that
    \[\prob(\tau_{p}(\epsilon) \geq x) = \prob ( \exists t\geq x: \exists \mathbf{s}\in \mathcal{I}: \norm{\Bar{P}_{\mathbf{s}, t} - P_{\mathbf{s}} }_1 \geq \epsilon ) \leq \sum_{\mathbf{s}\in \mathcal{I}} \sum_{t \geq x} \prob ( \norm{\Bar{P}_{\mathbf{s}, t} - P_{\mathbf{s}} }_1 \geq \epsilon),\]
    where the last inequality is from the union bound. Accordingly, for any $x \geq \alpha \abs{\mathcal{I}}^2 /(1-\alpha)^2$, we apply~\cref{cor: L1_dev} to get
    \begin{align*}
        \prob(\tau_{p}(\epsilon) \geq x) & \leq (2^{\abs{\omega(\mathbf{V})}} - 2) \abs{\mathcal{I}} \sum_{t \geq x} \exp\Big(\frac{-\sqrt{\alpha t} \epsilon^2}{2}\Big)\\
        & \leq (2^{\abs{\omega(\mathbf{V})}} - 2) \abs{\mathcal{I}} \int_{x-1}^{+\infty} \exp\Big(\frac{-\sqrt{\alpha t} \epsilon^2}{2}\Big) dx\\
        &= (2^{\abs{\omega(\mathbf{V})}} - 2) \abs{\mathcal{I}} \frac{8}{\alpha \epsilon^4} \exp\Big(\frac{-\sqrt{\alpha (x-1)} \epsilon^2}{2}\Big) \Big(\frac{\sqrt{\alpha (x-1)} \epsilon^2}{2} + 1 \Big)
    \end{align*}
    Let $\beta = \alpha \abs{\mathcal{I}}^2 /(1-\alpha)^2$. It follows that       
    \begin{align*}
        \expt[\tau_{p}(\epsilon)] &\leq \beta + 1 + \int_{\beta + 1}^{+\infty} \prob(\tau_{p}(\epsilon) \geq x) dx \\
        & \leq \beta + 1 + (2^{\abs{\omega(\mathbf{V})}} - 2) \abs{\mathcal{I}} \frac{64}{\alpha^2 \epsilon^8} \exp\Big(\frac{-\sqrt{\alpha \beta} \epsilon^2}{2}\Big) \Big(\frac{\alpha \beta \epsilon^4}{4} +\frac{3\sqrt{\alpha \beta} \epsilon^2}{2} + 3\Big):= g(\epsilon, \alpha).
    \end{align*}
    Taking $c(\epsilon) = \inf_{\alpha\in(0,1)} g(\epsilon, \alpha)$, we conclude the proof.
\end{proof}

\subsection{Proof of~\cref{th: upperbound}}
We decompose~\cref{th: upperbound} into~\cref{lemma: accuracy,lemma: ex_ub,lemma: boundA_P} and prove them in separate sections. 

\subsubsection{Accuracy of the Track-and-stop Causal Discovery Algorithm}
In this section, we prove that for any $\delta \in (0,1)$, the confidence level $1- \delta$ can be reached by the track-and-stop causal discovery algorithm (exact and practical version). The following concentration inequality is crucial in the proof. For an active learning setup with feedback drawn from Categorical distributions, a concentration bound on the empirical distribution is presented in Lemma 6 of~\cite{van2020optimal}. In the causal discovery problem, the actions space is $\mathcal{I}$, and the discrete support of feedback is $\omega(\mathbf{V})$. At each time $t$, for each intervention $\mathbf{s} \in \mathcal{I}$, recall $\Bar{P}_{\mathbf{s},t}$ is the empirical interventional distribution of $\mathbf{V}$ and $N_t(\mathbf{s})$ is the number of times the intervention $do(\mathbf{S} = \mathbf{s})$ is taken till $t$. For each intervention $\mathbf{s} \in \mathcal{I}$, the true interventional distribution is $P^{\mathcal{D}^*}_\mathbf{s}$.
\begin{lemma}[Concentration Inequality for Information Distance~\cite{van2020optimal}]~\label{lemma: concentration KL}
Let $x \geq \abs{\mathcal{I}} (\abs{\omega(V)}-1)$. Then for any $t > 0$,
    \[ \prob \left[ \sum_{\mathbf{s} \in \mathcal{I}} N_t(\mathbf{s}) \kl{\Bar{P}_{\mathbf{s}, t}}{P^{\mathcal{D}^*}_\mathbf{s}} \geq x \right ]  \leq \bigg (\frac{x \lceil x \ln t  + 1 \rceil 2 e}{\abs{\mathcal{I}} (\abs{\omega(V)}-1)}\bigg)^{\abs{\mathcal{I}} (\abs{\omega(V)}-1)}  \exp(1 -x).\]
\end{lemma}

\begin{lemma}\label{lemma: accuracy}
For the causal discovery problem with the MEC represented by CPDAG $\mathcal{C}$ and observational distributions being available, if the faithfulness assumption in~\cref{def: faith} holds, for both $\mathcal{A}_{\mathsf{I}}$ and $\mathcal{A}_{\mathsf{P}}$, $\prob(\psi \neq \mathcal{D}^*) \leq \delta$.   
\end{lemma}
\begin{proof}
The track-and-stop causal discovery algorithm keeps track of the most probable DAG
\[\mathcal{D}^*_t = \argmax_{\mathcal{D} \in [\mathcal{C}]} \sum_{\mathbf{s} \in \omega(\mathbf{S})}  N_t(\mathbf{s}, \mathbf{v}) \log P^{\mathcal{D}}(\mathbf{v}).\]
For $\mathcal{A}_{\mathsf{I}}$, at stopping time $\tau_\delta$, by the design if $\mathcal{D}^*_{\tau_{\delta}} \neq \mathcal{D}^*$, we have
\[d_{\tau_\delta} =  \min_{\mathcal{D} \in [\mathcal{C}]\setminus \mathcal{D}^*_{\tau_{\delta}}} \sum_{\mathbf{s} \in \mathcal{I}} N_{\tau_{\delta}}(\mathbf{s})  \kl{\Bar{P}_{{\tau_{\delta}},\mathbf{s}}}{P^{\mathcal{D}}_{\mathbf{s}}} \leq \sum_{\mathbf{s} \in \mathcal{I}} N_{\tau_{\delta}}(\mathbf{s})  \kl{\Bar{P}_{{\tau_{\delta}},\mathbf{s}}}{P^{\mathcal{D}^*}_{\mathbf{s}}}.\]
Then we apply~\cref{lemma: concentration KL} to get
\begin{equation} \label{lbpreq 1}
    \prob \left[ \psi \neq \mathcal{D}^* \right ] \leq \prob \bigg [ d_{\tau_{\delta}} \leq \sum_{\mathbf{s} \in \mathcal{I}} N_{\tau_{\delta}}(\mathbf{s})  \kl{\Bar{P}_{{\tau_{\delta}},\mathbf{s}}}{P^{\mathcal{D}^*}_{\mathbf{s}}} \bigg ] \leq f_{\tau_{\delta}} (d_{\tau}) \leq \delta,
\end{equation}      
where the last inequality is due to the termination condition of the algorithm. 

With $\mathcal{A}_{\mathsf{P}}$, instead of searching $\mathcal{D}^*$ in $[\mathcal{\mathcal{C}}]$, we search $\tupdef{\mathsf{C}^*(\mathbf{S})}{\mathbf{S} \in \mathcal{S}}$ in the space $\tupdef{\mathsf{CF}(\mathbf{S})}{\mathbf{S} \in\mathcal{S}}$. Recall $Z_t(\mathbf{S})=  \min_{\mathsf{C}(\mathbf{S}) \neq \mathsf{C}^*_t(\mathbf{S})} \sum_{\mathbf{s} \in \omega(\mathbf{S})} N_t(\mathbf{s})  \kl{P^{\mathsf{C}(\mathbf{S})}_{\mathbf{s}, t}}{P^{\mathsf{C}^*_t(\mathbf{S})}_{\mathbf{s}}}$. As a matter of fact,
\begin{equation}\label{appeq: d_t}
\begin{split}
    d_t &= \min_{\mathbf{S}\in\mathcal{S}} Z_t (\mathbf{S}) + \sum_{\mathbf{S}\in\mathcal{S}} \sum_{\mathbf{s} \in \omega(\mathbf{S})}  N_t(\mathbf{s}) \kl{\Bar{P}_{\mathbf{s}, t}}{P^{\mathsf{C}^*_t(\mathbf{S})}_{\mathbf{s}}}\\
    &= \min_{\tupdef{\mathsf{C}_{\tau_{\delta}}^*(\mathbf{S})}{\mathbf{S} \in \mathcal{S}} \neq \tupdef{\mathsf{C}^*(\mathbf{S})}{\mathbf{S} \in \mathcal{S}}} \sum_{\mathbf{s} \in \mathcal{I}} N_t(\mathbf{s}) \kl{\Bar{P}_{\mathbf{s}, t}}{P^{\mathsf{C}^*_t(\mathbf{S})}_{\mathbf{s}}}.
\end{split}    
\end{equation}
If $\tupdef{\mathsf{C}_{\tau_{\delta}}^*(\mathbf{S})}{\mathbf{S} \in \mathcal{S}} \neq \tupdef{\mathsf{C}^*(\mathbf{S})}{\mathbf{S} \in \mathcal{S}}$, following a similar reasoning, it can be seen~\eqref{lbpreq 1} still holds. We conclude the proof.
\end{proof}

\subsubsection{Asymptotic performance of Exact Algorithm}
\begin{lemma}\label{lemma: ex_ub}
    For the causal discovery problem, suppose the MEC represented by CPDAG $\mathcal{C}$ and observational distributions are available. If the faithfulness assumption in~\cref{def: faith} holds, for the exact algorithm $\mathcal{A}_{\mathsf{I}}$, we have $\prob(\tau_{\delta} = \infty) = 0$ and 
    \[\lim_{\delta \rightarrow 0} \frac{\log (1/\delta)}{ \expt[\tau_{\delta}]} = {c}(\mathcal{D}^*).\]
\end{lemma}

\begin{proof}    
Let an arbitrary intervention distribution tuple be $\mathcal{P} = \tupdef{P_\mathbf{s} }{\mathbf{s} \in \mathcal{I}}$. By the continuity of KL-divergence, there exists a small enough constant $c >0$ such that if $\norm{P_{\mathbf{s}} - P^{\mathcal{D^*}}_{\mathbf{s}} }_1 \leq c$ holds for all $\mathbf{s} \in \mathcal{I}$, for any $\mathcal{D} \in [\mathcal{C}] \setminus [\mathcal{D}^*]$, it satisfies that
\begin{equation}\label{preq 1: l1_close}
    \forall \mathbf{s} \in \mathcal{I}: \kl{P_{\mathbf{s}}}{P^{\mathcal{D}^*}_{\mathbf{s}}} \leq \kl{P_{\mathbf{s}}}{P^{\mathcal{D}}_{\mathbf{s}}} \text{ and } \exists \mathbf{s} \in \mathcal{I}: \kl{P_{\mathbf{s}}}{P^{\mathcal{D}^*}_{\mathbf{s}}} < \kl{P_{\mathbf{s}}}{P^{\mathcal{D}}_{\mathbf{s}}}.
\end{equation}
Recall that at time $t$, the track-and-stop causal discovery algorithm tracks the most probable DAG
\[\mathcal{D}^*_t \in \argmax_{\mathcal{D} \in [\mathcal{C}]} \sum_{\mathbf{s} \in \mathcal{I}}  N_t(\mathbf{s}, \mathbf{v}) \log P^{\mathcal{D}}_{\mathbf{s}}(\mathbf{v}).\] 
As a matter of fact,
\[\mathcal{D}^*_t \in \argmin_{\mathcal{D} \in [\mathcal{C}]} \sum_{\mathbf{s} \in \mathcal{I}}  N_t(\mathbf{s}) \kl{\Bar{P}_{\mathbf{s},t}}{P^{\mathcal{D}}_{\mathbf{s}}}.\]
Thus, if $\forall \mathbf{s} \in \mathcal{I}: \norm{\Bar{P}_{\mathbf{s},t} - P^{\mathcal{D^*}}_{\mathbf{s}} }_1 \leq c$, according to conditions in~\eqref{preq 1: l1_close}, $\mathcal{D}^*_t = \mathcal{D}^*$ can be uniquely determined.

For $\epsilon \in (0,c]$, define time $\tau_{p}(\epsilon) = \max \setdef{t \in \natural_{>0}}{ \exists \mathbf{s}: \norm{\Bar{P}_{\mathbf{s}, t} - P_{\mathbf{s}} }_1 \geq \epsilon}$. Therefore, for any $T \geq \tau_{p} (\epsilon)$, $\mathcal{D}^*_T = \mathcal{D}^*$. As a result,
\begin{equation}\label{preq 1: dT}
    d_T =  \min_{\mathcal{D} \in [\mathcal{C}]\setminus \mathcal{D}^*_t} \sum_{\mathbf{s} \in \mathcal{I}} N_T(\mathbf{s})  \kl{\Bar{P}_{\mathbf{s}, T}}{P^{\mathcal{D}}_{\mathbf{s}}}
    = \min_{ \mathcal{D} \in [\mathcal{C}] \setminus \mathcal{D}^* } \sum_{\mathbf{s} \in \mathcal{I}} N_T(\mathbf{s})  \KL{\Bar{P}_{\mathbf{s}, T}}{P^{\mathcal{D}}_{\mathbf{s}}}.
\end{equation}
It follows from~\cref{lemma: track} that
\begin{align}
    \eqref{preq 1: dT} & \geq \min_{ \mathcal{D} \in [\mathcal{C}] \setminus \mathcal{D}^* } \sum_{\mathbf{s} \in \mathcal{I}} \sum_{t=1}^{T} \alpha_{\mathbf{s},t} \KL{\Bar{P}_{\mathbf{s}, T}}{P^{\mathcal{D}}_{\mathbf{s}}} - \abs{\mathcal{I}} (\abs{\mathcal{I}} - 1) (\sqrt{t} + 2) D \nonumber \\
    &\geq  \min_{ \mathcal{D} \in [\mathcal{C}] \setminus \mathcal{D}^* } \sum_{t=\tau_p(\epsilon)+1}^T \sum_{\mathbf{s} \in \mathcal{I}}  \alpha_{\mathbf{s},t} \kl{\Bar{P}_{\mathbf{s}, T}}{P^{\mathcal{D}}_{\mathbf{s}}} - \abs{\mathcal{I}} (\abs{\mathcal{I}} - 1) (\sqrt{t} + 2)D \nonumber\\
    &\geq  \min_{ \mathcal{D} \in [\mathcal{C}] \setminus \mathcal{D}^* } \sum_{t=\tau_p(\epsilon)+1}^T \sum_{\mathbf{s} \in \mathcal{I}}  \alpha_{\mathbf{s},t} \kl{\Bar{P}_{\mathbf{s}, t}}{P^{\mathcal{D}}_{\mathbf{s}}} - 2[T-\tau_p(\epsilon)] u(\epsilon)- \abs{\mathcal{I}} (\abs{\mathcal{I}} - 1) (\sqrt{t} + 2)D \nonumber\\
    &\geq \sum_{t=\tau_p(\epsilon)+1}^T \min_{ \mathcal{D} \in [\mathcal{C}] \setminus \mathcal{D}^* } \sum_{\mathbf{s} \in \mathcal{I}}  \alpha_{\mathbf{s},t} \kl{\Bar{P}_{\mathbf{s}, t}}{P^{\mathcal{D}}_{\mathbf{s}}} - 2[T-\tau_p(\epsilon)] u(\epsilon) - \abs{\mathcal{I}} (\abs{\mathcal{I}} - 1) (\sqrt{t} + 2) D, \label{preq 1: dT2}
\end{align}
where
\[ u(\epsilon) = \sup_{\tupdef{P_{\mathbf{s}}}{\mathbf{s} \in \mathcal{I}}} \bigg \{ { \max_{\mathcal{D} \in [\mathcal{C}]} \abs{\kl{P_\mathbf{s}}{P^{\mathcal{D}}_\mathbf{s}} - \kl{P^{\mathcal{D}^*}_\mathbf{s}}{P^{\mathcal{D}}_\mathbf{s}}} }: {\norm{P_{\mathbf{s}} - P^{\mathcal{D^*}}_{\mathbf{s}} }_1\leq \epsilon, \forall \mathbf{s} \in \mathcal{I} } \bigg \}.\]
Recall that we define $\mathcal{D}'_t \in \argmin_{\mathcal{D} \in [\mathcal{C}]\setminus \mathcal{D}^*_t} \sum_{\mathbf{s} \in \mathcal{I}} \alpha_{\mathbf{s}, t} \kl{\Bar{P}_{\mathbf{s}, t}}{P^{\mathcal{D}}_{\mathbf{s}}}$. With~\cref{lemma: regret}, we have
\begin{align}
    \eqref{preq 1: dT2} &\geq \sum_{t=\tau_p(\epsilon)+1}^T \sum_{\mathbf{s} \in \mathcal{I}}  \alpha_{\mathbf{s},t} \kl{\Bar{P}_{\mathbf{s}, t}}{P^{\mathcal{D}'_t}_{\mathbf{s}}} - 2[T-\tau_p(\epsilon)] u(\epsilon) - \abs{\mathcal{I}} (\abs{\mathcal{I}} - 1) (\sqrt{t} + 2)D \nonumber\\
    & \geq \max_{\mathbf{s} \in \mathcal{I}} \sum_{t=\tau_p(\epsilon)+1}^T \KL{\Bar{P}_{\mathbf{s}, t}}{P^{\mathcal{D}'_t}_{\mathbf{s}}} - \supscr{R}{AH}_T - \tau_p(\epsilon) D - 2[T-\tau_p(\epsilon)] u(\epsilon) - \abs{\mathcal{I}} (\abs{\mathcal{I}} - 1) (\sqrt{t} + 2)D \nonumber\\ 
    & \geq \max_{\mathbf{s} \in \mathcal{I}} \sum_{t=\tau_p(\epsilon)+1}^T \KL{{P}^{\mathcal{D}^*}_{\mathbf{s}}}{P^{\mathcal{D}'_t}_{\mathbf{s}}} - \supscr{R}{AH}_T - \tau_p(\epsilon) D - 3[T-\tau_p(\epsilon)] u(\epsilon) - \abs{\mathcal{I}} (\abs{\mathcal{I}} - 1) (\sqrt{t} + 2)D \nonumber\\
    &= \max_{\mathbf{s} \in \mathcal{I}} \sum_{\mathcal{D} \in [\mathcal{C}]} N_{\tau_p(\epsilon) : T}(\mathcal{D}) \KL{{P}^{\mathcal{D}^*}_{\mathbf{s}}}{P^{\mathcal{D}}_{\mathbf{s}}} - \supscr{R}{AH}_T - \tau_p(\epsilon) D - 3[T-\tau_p(\epsilon)] u(\epsilon) - \abs{\mathcal{I}} (\abs{\mathcal{I}} - 1) (\sqrt{t} + 2)D, \label{preq 1: dT3}
\end{align}
where $N_{\tau_p(\epsilon): T}(\mathcal{D}) = \sum_{t=\tau_p(\epsilon)+1}^T \indicator{\mathcal{D}'_t = \mathcal{D}}$. With~\cref{lemma: duality}, we have 
\begin{align*}
    \max_{\mathbf{s} \in \mathcal{I}} \sum_{\mathcal{D} \in [\mathcal{C}]} N_{\tau_p(\epsilon) : T}(\mathcal{D}) \KL{{P}^{\mathcal{D}^*}_{\mathbf{s}}}{P^{\mathcal{D}}_{\mathbf{s}}} &\geq [T-\tau_p(\epsilon)] \inf_{\boldsymbol{w}\in \Delta([\mathcal{C}]\setminus \mathcal{D}^*)} \max_{\mathbf{s} \in \mathcal{I}}  \sum_{\mathcal{D} \in [\mathcal{C}] \setminus \mathcal{D}^*} w_{\mathcal{D}} \KL{P^{\mathcal{D}^*}_{\mathbf{s}}}{P^{\mathcal{D}}_{\mathbf{s}}} \\
    &= [T-\tau_p(\epsilon)] c(\mathcal{D}^*).
\end{align*}
Plugging the above result into~\eqref{preq 1: dT3}, we get
\begin{equation}
    \eqref{preq 1: dT3} \geq [T-\tau_p(\epsilon)] [ c(\mathcal{D}^*) - 3u(\epsilon)] - \supscr{R}{AH}_T - \tau_p(\epsilon) D - \abs{\mathcal{I}} (\abs{\mathcal{I}} - 1)(\sqrt{t} + 2)D := \underline{d}_t. \label{preq 1: dT4}
\end{equation}
Define time
\[ \Bar{\tau}_{\delta} = \max_{\tau} \Big \{\tau \in \natural_{>0}: d_{\tau} \geq \abs{\mathcal{I}} (\abs{\omega(V)}-1), f_{\tau}(\underline{d}_{\tau}) \leq \delta \Big\}.\]
According to the termination condition of the causal discovery algorithm, the algorithm terminates at ${\tau}_{\delta} \leq \Bar{\tau}_{\delta}$. Since~\cref{corr: boundedexpt} shows that $\expt[\tau_p(\epsilon)]$ is bounded by a constant, which means $\prob(\expt[\tau_p(\epsilon)] = \infty) = 0$, we have
\[\prob({\tau}_{\delta} = \infty) \leq \prob(\Bar{\tau}_{\delta} = \infty) = 0.\]
With~\cref{lemma: adh}, notice that $\eqref{preq 1: dT4}=T[c(\mathcal{D}^*)-3u(\epsilon)]+o(T)$ and $f_t(x)$ is dominated by $\exp(-x)$. We have for any $\epsilon \in (0,c]$
\[\lim_{\delta \rightarrow 0} \frac{\log (1/\delta)}{ \expt[\tau_{\delta}]} \geq \frac{\log (1/\delta)}{ \expt[\Bar{\tau}_{\delta}]} = c(\mathcal{D}^*)-3u(\epsilon),\]
The continuity of KL-divergence ensures that $\lim_{\epsilon \rightarrow 0} u(\epsilon) = 0$. Then we have that
\[\lim_{\delta \rightarrow 0} \frac{\log (1/\delta)}{ \expt[\tau_{\delta}]} \geq c(\mathcal{D}^*).\]
Combining with the lower bound result in~\cref{th: lowerbound}, we conclude the proof.
\end{proof}

\subsubsection{Asymptotic performance of Practical Algorithm}
\begin{lemma}\label{lemma: cons_est}
    For the causal discovery problem with $\omega(\mathbf{S}) \subseteq \mathcal{I}$ contains all interventions on the node set $\mathbf{S}$, we have ${c}(\mathcal{D}^*) \geq \underline{c}(\mathcal{D}^*).$
\end{lemma}
\begin{proof}
    Since the set of interventions $\mathcal{I}$ can be partitioned into interventions on different node sets, we have $\mathcal{I} = \union_{\mathbf{S} \in \mathcal{S}} \omega(\mathbf{S})$ and $\omega(\mathbf{S}) \intersection \omega(\mathbf{S}') = \emptyset$ for $\mathbf{S} \neq \mathbf{S}'$. Accordingly, for every $\mathcal{D} \in [\mathcal{C}] \setminus \mathcal{D}^*$, there exists at least one edge cut that has a different configuration compared with $\mathcal{D}^*$
\begin{align*}
    \min_{ \mathcal{D} \in [\mathcal{C}] \setminus \mathcal{D}^* } \sum_{\mathbf{s} \in \mathcal{I}}  \alpha_{\mathbf{s}} \KL{P^*_{\mathbf{s}}}{P^{\mathcal{D}}_{\mathbf{s}}} =&\min_{ \mathcal{D} \in [\mathcal{C}] \setminus \mathcal{D}^* } \sum_{\mathbf{S} \in \mathcal{S}} \sum_{\mathbf{s} \in  \omega(\mathbf{S})}  \alpha_{\mathbf{s}} \KL{P^*_{\mathbf{s}}}{P^{\mathcal{D}}_{\mathbf{s}}} \\
    \geq &  \min_{\mathbf{S} \in \mathcal{S}} \min_{\mathsf{C}(\mathbf{S}) \neq \mathsf{C}^*(\mathbf{S})} \sum_{\mathbf{s} \in \omega(\mathbf{S})} \alpha_{\mathbf{s}} \KL{P^*_{\mathbf{s}}}{P^{\mathsf{C}(\mathbf{S})}_{\mathbf{s}}},
\end{align*}
for any $\boldsymbol{\alpha} \in \Delta(\mathcal{I})$. Thus we get
\begin{align*}
     {c}(\mathcal{D}^*) &= \sup_{\boldsymbol{\alpha} \in \Delta(\mathcal{I})} \min_{ \mathcal{D} \in [\mathcal{C}] \setminus \mathcal{D}^* } \sum_{\mathbf{s} \in \mathcal{I}}  \alpha_{\mathbf{s}} \KL{P^*_{\mathbf{s}}}{P^{\mathcal{D}}_{\mathbf{s}}} \\
     &\leq \sup_{\boldsymbol{\alpha} \in \Delta(\mathcal{I})} \min_{\mathbf{S} \in \mathcal{S}} \min_{\mathsf{C}(\mathbf{S}) \neq \mathsf{C}^*(\mathbf{S})}  \sum_{\mathbf{s} \in \omega(\mathbf{S})}  \alpha_{\mathbf{s}} \kl{P^{\mathcal{D}^*\!\!}_{\mathbf{s}}}{P^{\mathsf{C}(\mathbf{S})}_{\mathbf{s}}} = \underline{c}(\mathcal{D}^*).
\end{align*}
We reached the result.
\end{proof}

\begin{lemma}\label{lemma: boundA_P}
    For the causal discovery problem, suppose the MEC represented by CPDAG $\mathcal{C}$ and observational distributions are available. If the faithfulness assumption in~\cref{def: faith} holds, for the practical algorithm $\mathcal{A}_{\mathsf{P}}$, we have $\prob(\tau_{\delta} = \infty) = 0$ and 
    \[\lim_{\delta \rightarrow 0} \frac{\log (1/\delta)}{ \expt[\tau_{\delta}]} = \underline{c}(\mathcal{D}^*).\]
\end{lemma}
\begin{proof}[Sketch of Proof]
With the practical algorithm $\mathcal{A}_{\mathsf{P}}$, instead of searching $\mathcal{D}^*$ in $[\mathcal{\mathcal{C}}]$, we search $\tupdef{\mathsf{C}^*(\mathbf{S})}{\mathbf{S} \in \mathcal{S}}$ in the space $\tupdef{\mathsf{CF}(\mathbf{S})}{\mathbf{S} \in\mathcal{S}}$. Recall 
\[\mathsf{C}'_t(\mathbf{S}) \in \argmin_{\mathsf{C}(\mathbf{S}) \neq \mathsf{C}_t^*(\mathbf{S})} \sum_{\mathbf{s} \in \omega(\mathbf{S})} \xi^{\mathbf{S}}_{\mathbf{s}, t} \kl{\Bar{P}_{\mathbf{s}, t}}{P^{\mathsf{C}(\mathbf{S})}_{\mathbf{s}}},\] 
and $\tau_{p}(\epsilon) = \max \setdef{t \in \natural_{>0}}{ \exists \mathbf{s}: \norm{\Bar{P}_{\mathbf{s}, t} - P_{\mathbf{s}} }_1 \geq \epsilon}$. For each $\mathbf{S} \in \mathcal{S}$, we have
    \begin{equation}\label{preq3: c_t}
        \begin{split}
            T c_T (\mathbf{S}) =&\sum_{t= 1}^T \sum_{\mathbf{s} \in \omega(\mathbf{S})} \xi^{\mathbf{S}}_{\mathbf{s}, i} \kl{\Bar{P}_{\mathbf{s}, t}}{P^{\mathsf{C}'_t(\mathbf{S})}_{\mathbf{s}}} \\
            \geq& \sum_{t=\tau_p(\epsilon)+1}^T \sum_{\mathbf{s} \in \omega(\mathbf{S})} \xi^{\mathbf{S}}_{\mathbf{s}, i} \kl{\Bar{P}_{\mathbf{s}, t}}{P^{\mathsf{C}'_t(\mathbf{S})}_{\mathbf{s}}} - \tau_p(\epsilon) D \\
            \geq& \max_{\mathbf{s} \in \omega(\mathbf{S})} \sum_{t=\tau_p(\epsilon)+1}^T \KL{\Bar{P}_{\mathbf{s}, t}}{P^{\mathsf{C}'_t(\mathbf{S})}_{\mathbf{s}}} - \supscr{R}{AH}_T - \tau_p(\epsilon) D \\
            \geq& \max_{\mathbf{s} \in \omega(\mathbf{S})} \sum_{t=\tau_p(\epsilon)+1}^T \KL{P^{\mathcal{D}^*}_{\mathbf{s}}}{P^{\mathsf{C}'_t(\mathbf{S})}_{\mathbf{s}}} - \supscr{R}{AH}_T - \tau_p(\epsilon) D - [T-\tau_p(\epsilon)] u(\epsilon),
        \end{split}
    \end{equation}

where we apply~\cref{lemma: adh} in the second inequality. Let $N_{\tau_p(\epsilon): T}(\mathsf{C}(\mathbf{S})) = \sum_{t=\tau_p(\epsilon)+1}^T \indicator{\mathsf{C}'_t(\mathbf{S}) = \mathsf{C}(\mathbf{S})}$. We apply the second inequality in~\cref{lemma: duality} to get 
\begin{align}
    &\max_{\mathbf{s} \in \omega(\mathbf{S})} \sum_{t=\tau_p(\epsilon)+1}^T \KL{P^{\mathcal{D}^*}_{\mathbf{s}}}{P^{\mathsf{C}'_t(\mathbf{S})}_{\mathbf{s}}}  \nonumber\\
    =& \max_{\mathbf{s} \in \omega(\mathbf{S})} \sum_{\mathsf{C}(\mathbf{S}) \in \mathsf{CF}(\mathbf{S})} N_{\tau_p(\epsilon): T}(\mathsf{C}(\mathbf{S})) \KL{P^{\mathcal{D}^*}_{\mathbf{s}}}{P^{\mathsf{C}(\mathbf{S})}_{\mathbf{s}}} \nonumber\\
    \geq& [T-\tau_p(\epsilon)] \inf_{\boldsymbol{\zeta}^{\mathbf{S}} \in \Delta(\mathsf{CF}(\mathbf{S})\setminus \mathsf{C}^*(\mathbf{S}))} \max_{\mathbf{s} \in \omega(\mathbf{S})} \sum_{\mathsf{C}(\mathbf{S}) \in \mathsf{CF}(\mathbf{S}) \setminus \mathsf{C}^*(\mathbf{S})} \zeta^{\mathbf{S}}_{\mathsf{C}(\mathbf{S})} \kl{P^{\mathsf{C}^*(\mathbf{S})}_{\mathbf{s}}}{P^{\mathsf{C}(\mathbf{S})}_{\mathbf{s}, t}} \nonumber\\
    =&[T-\tau_p(\epsilon)] [ c_{\mathbf{S}}(\mathcal{D}^*) - u(\epsilon)], \nonumber
\end{align}
where $c_{\mathbf{S}}(\mathcal{D}^*)$ is defined in~\eqref{eqdef: local}. Plugging the above result into~\eqref{preq3: c_t}, we get
\begin{equation}\label{preq3: fn}
    T c_T (\mathbf{S}) \geq  [T-\tau_p(\epsilon)] [ c_{\mathbf{S}}(\mathcal{D}^*) - u(\epsilon)]  - \supscr{R}{AH}_T - \tau_p(\epsilon) D.
\end{equation}
Since $\supscr{R}{AH}_T \leq \sqrt{D T \ln K} + D\left(\frac{4}{3} \ln K + 2\right)$, \eqref{preq3: fn} indicates $c_t(\mathbf{S}) / t \rightarrow c_{\mathbf{S}}(\mathcal{D}^*) - u(\epsilon)$ as $t \rightarrow \infty$. Furthermore, since $ \gamma_{\mathbf{S},t} \propto {1}/{c_t (\mathbf{S})} $, we can define a stopping time
\begin{equation} \label{appdef: taupgam}
    \tau_{p,\gamma}(\epsilon) := \max \Bigsetdef{t \geq \tau_p(\epsilon)}{ \sum_{\mathbf{S} \in \mathcal{S}} \abs{\gamma_{\mathbf{S},t} - \gamma^*_{\mathbf{S}}}  \geq \epsilon}.
\end{equation}
With $\expt[\tau_{P}(\epsilon)] \leq c(\epsilon)$ according to~\cref{corr: boundedexpt}, we have $\expt[\tau_{P,\gamma}(\epsilon)] \leq c'(\epsilon)$ for some $c'(\epsilon)\leq \infty$.

Similar to the proof of~\cref{lemma: ex_ub}, by the continuity of KL-divergence, there exists a small enough constant $c >0$ such that if $\forall \mathbf{s} \in \mathcal{I}: \norm{\Bar{P}_{\mathbf{s},t} - P^{\mathcal{D^*}}_{\mathbf{s}} }_1 \leq c$ holds for all $\mathbf{s} \in \mathcal{I}$, each $\mathsf{C}^*_t(\mathbf{S}) = \mathsf{C}^*(\mathbf{S})$ for all $\mathbf{S} \in \mathcal{S}$ can be uniquely determined. 
 Therefore, for any $\epsilon \in (0,c]$, , $\forall \mathbf{S} \in \mathcal{S}: \mathsf{C}^*_T(\mathbf{S}) = \mathsf{C}^*(\mathbf{S})$, if $T \geq \tau_{p} (\epsilon)$. It follows from~\eqref{appeq: d_t} that for $T \geq \tau_{p} (\epsilon)$,
\begin{align}
    d_T &= \min_{\tupdef{\mathsf{C}(\mathbf{S})}{\mathbf{S} \in \mathcal{S}} \neq \tupdef{\mathsf{C}_t^*(\mathbf{S})}{\mathbf{S} \in \mathcal{S}}} \sum_{\mathbf{s} \in \mathcal{I}} N_t(\mathbf{s}) \kl{\Bar{P}_{\mathbf{s}, t}}{P^{\mathsf{C}(\mathbf{S})}_{\mathbf{s}}} \nonumber\\
    &= \min_{\tupdef{\mathsf{C}(\mathbf{S})}{\mathbf{S} \in \mathcal{S}} \neq \tupdef{\mathsf{C}^*(\mathbf{S})}{\mathbf{S} \in \mathcal{S}}} \sum_{\mathbf{s} \in \mathcal{I}} N_t(\mathbf{s}) \kl{\Bar{P}_{\mathbf{s}, t}}{P^{\mathsf{C}(\mathbf{S})}_{\mathbf{s}}} \nonumber\\
     &\geq \min_{\mathbf{S} \in \mathcal{S}} \min_{\mathsf{C}(\mathbf{S}) \neq \mathsf{C}^*(\mathbf{S})} \sum_{\mathbf{s} \in \omega(\mathbf{S})}  N_T(\mathbf{s})  \KL{\Bar{P}_{\mathbf{s}, T}}{P^{\mathsf{C}(\mathbf{S})}_{\mathbf{s}}} \nonumber\\
    & \geq \min_{\mathbf{S} \in \mathcal{S}}  \min_{\mathsf{C}(\mathbf{S}) \neq \mathsf{C}^*(\mathbf{S})} \Bigg [ \sum_{\mathbf{s} \in \omega(\mathbf{S})}  \sum_{t=1}^{T} \alpha_{\mathbf{s},t} \KL{\Bar{P}_{\mathbf{s}, T}}{P^{\mathsf{C}(\mathbf{S}}_{\mathbf{s}}} - \abs{\omega(\mathbf{S})} (\abs{\mathcal{I}} - 1) (\sqrt{t} + 2) D \Bigg ] \nonumber\\
    &\geq \min_{\mathbf{S} \in \mathcal{S}}  \min_{\mathsf{C}(\mathbf{S}) \neq \mathsf{C}^*(\mathbf{S})} \sum_{t=\tau_{p,\gamma}(\epsilon)+1}^T  \sum_{\mathbf{s} \in \omega(\mathbf{S})} \alpha_{\mathbf{s},t} \kl{\Bar{P}_{\mathbf{s}, T}}{P^{\mathsf{C}(\mathbf{S}}_{\mathbf{s}}} - \abs{\mathcal{I}} (\abs{\mathcal{I}} - 1) (\sqrt{t} + 2) D \label{preq3: dt1}
\end{align}
where we apply~\cref{lemma: track} in the second inequality. Since $\alpha_{\mathbf{s}} = \gamma_{\mathbf{S},t} \xi^{\mathbf{S}}_{\mathbf{s},t}$
\begin{align}
    &\min_{\mathbf{S} \in \mathcal{S}}  \min_{\mathsf{C}(\mathbf{S}) \neq \mathsf{C}^*(\mathbf{S})} \sum_{t=\tau_{p,\gamma}(\epsilon)+1}^T  \sum_{\mathbf{s} \in \omega(\mathbf{S})} \alpha_{\mathbf{s},t} \kl{\Bar{P}_{\mathbf{s}, T}}{P^{\mathsf{C}(\mathbf{S}}_{\mathbf{s}}} \nonumber\\
    = & \min_{\mathbf{S} \in \mathcal{S}}  \min_{\mathsf{C}(\mathbf{S}) \neq \mathsf{C}^*(\mathbf{S})} \sum_{t=\tau_{p,\gamma}(\epsilon)+1}^T \gamma_{\mathbf{S},t} \sum_{\mathbf{s} \in \omega(\mathbf{S})} \xi^{\mathbf{S}}_{\mathbf{s},t} \kl{\Bar{P}_{\mathbf{s}, T}}{P^{\mathsf{C}(\mathbf{S}}_{\mathbf{s}}} \nonumber\\
    \geq & \min_{\mathbf{S} \in \mathcal{S}}  \gamma_{\mathbf{S}}^* \min_{\mathsf{C}(\mathbf{S}) \neq \mathsf{C}^*(\mathbf{S})} \sum_{t=\tau_{p,\gamma}(\epsilon)+1}^T  \sum_{\mathbf{s} \in \omega(\mathbf{S})} \xi^{\mathbf{S}}_{\mathbf{s},t} \kl{\Bar{P}_{\mathbf{s}, T}}{P^{\mathsf{C}(\mathbf{S}}_{\mathbf{s}}} - [T-\tau_{p,\gamma}(\epsilon)]\epsilon D \nonumber \\
    \geq & \min_{\mathbf{S} \in \mathcal{S}}  \gamma_{\mathbf{S}}^* \min_{\mathsf{C}(\mathbf{S}) \neq \mathsf{C}^*(\mathbf{S})} \sum_{t=\tau_{p,\gamma}(\epsilon)+1}^T  \sum_{\mathbf{s} \in \omega(\mathbf{S})} \xi^{\mathbf{S}}_{\mathbf{s},t} \kl{\Bar{P}_{\mathbf{s}, t}}{P^{\mathsf{C}(\mathbf{S}}_{\mathbf{s}}} -2[T-\tau_{p,\gamma}(\epsilon)] u(\epsilon) - [T-\tau_{p,\gamma}(\epsilon)]\epsilon D \label{preq3: dt}
\end{align}
where the first inequality is due to definition~\eqref{appdef: taupgam}. With~\eqref{preq3: fn}, we have
\begin{align}
     \min_{\mathsf{C}(\mathbf{S}) \neq \mathsf{C}^*(\mathbf{S})} \sum_{t=\tau_{p,\gamma}(\epsilon)+1}^T  \sum_{\mathbf{s} \in \omega(\mathbf{S})} \xi^{\mathbf{S}}_{\mathbf{s},t} \kl{\Bar{P}_{\mathbf{s}, T}}{P^{\mathsf{C}(\mathbf{S}}_{\mathbf{s}}} &\geq T c_T(\mathbf{S})  -  \tau_{p,\gamma}(\epsilon) D \nonumber\\
    &\geq [T-\tau_{p,\gamma}(\epsilon)] [ c_{\mathbf{S}}(\mathcal{D}^*) - u(\epsilon)]  - \supscr{R}{AH}_T - \tau_p(\epsilon) D  -  \tau_{p,\gamma}(\epsilon) D \label{preq3: dt2}
\end{align}
Putting together~\eqref{preq3: dt1},~\eqref{preq3: dt} and~\eqref{preq3: dt2}, we apply the third equality in~\cref{lemma: duality} to get
\[d_T \geq [T-\tau_{p,\gamma}(\epsilon)] [ \underline{c}(\mathcal{D}^*) - 3 u(\epsilon) -\epsilon D]  - \supscr{R}{AH}_T - \tau_{p,\gamma}(\epsilon) D  -  \tau_{p,\gamma}(\epsilon) D - \abs{\mathcal{I}} (\abs{\mathcal{I}} - 1) (\sqrt{t} + 2) D:= \underline{d}_t.\]
The remaining proof is similar to that of~\cref{lemma: ex_ub}. Define time
\[ \Bar{\tau}_{\delta} = \max_{\tau} \Big \{\tau \in \natural_{>0}: d_{\tau} \geq \abs{\mathcal{I}} (\abs{\omega(V)}-1), f_{\tau}(\underline{d}_{\tau}) \leq \delta \Big\}.\]
According to the termination condition of the causal discovery algorithm, the algorithm terminates at ${\tau}_{\delta} \leq \Bar{\tau}_{\delta}$. Since~\cref{corr: boundedexpt} shows that $\expt[\tau_p(\epsilon)]$ is bounded, so is $\expt[\tau_{P,\gamma}(\epsilon)]$. Accordingly, $\prob(\expt[\tau_p(\epsilon)] = \infty) = 0$, and we have
\[\prob({\tau}_{\delta} = \infty) \leq \prob(\Bar{\tau}_{\delta} = \infty) = 0.\]
With~\cref{lemma: adh}, notice that $\eqref{preq 1: dT4}=T[\underline{c}(\mathcal{D}^*)-3u(\epsilon) - \epsilon D]+o(T)$ and $f_t(x)$ is dominated by $\exp(-x)$. For any $\epsilon \in (0,c]$, it satisfies that
\[\lim_{\delta \rightarrow 0} \frac{\log (1/\delta)}{ \expt[\tau_{\delta}]} \geq \frac{\log (1/\delta)}{ \expt[\Bar{\tau}_{\delta}]} = \underline{c}(\mathcal{D}^*)-3u(\epsilon) -\epsilon D,\]
The continuity of KL-divergence ensures that $\lim_{\epsilon \rightarrow 0} u(\epsilon) = 0$. Then we have that
\[\lim_{\delta \rightarrow 0} \frac{\log (1/\delta)}{ \expt[\tau_{\delta}]} \geq \underline{c}(\mathcal{D}^*).\]
We conclude the proof.

\end{proof}

\subsection{Additional Experiment with limited observational data}

Although our setup requires access to the true observational distribution and focuses on minimizing the number of interventional samples to learn the true Directed Acyclic Graph (DAG) from the Markov equivalence class or CPDAG. In cases where observational data is limited—i.e., we don't have access to the true interventional distribution—the CPDAG can't be learned accurately and might differ from the true CPDAG. In this scenario, we might not be able to learn the true DAG using any existing causal discovery algorithm, but we can still run and test our algorithm and compare it with existing baselines. We test our algorithm on randomly sampled DAGs with $10$ nodes and a density parameter $\rho = 0.2$ starting with incorrect CPDAGs. All the discovery algorithms are tested on 50 randomly generated DAGs, similar to the experiments section in the main paper. The results are plotted in Figure \ref{add_exp}.

\begin{figure}[H]
\centering
\includegraphics[width=8.2cm,height = 4.5cm]{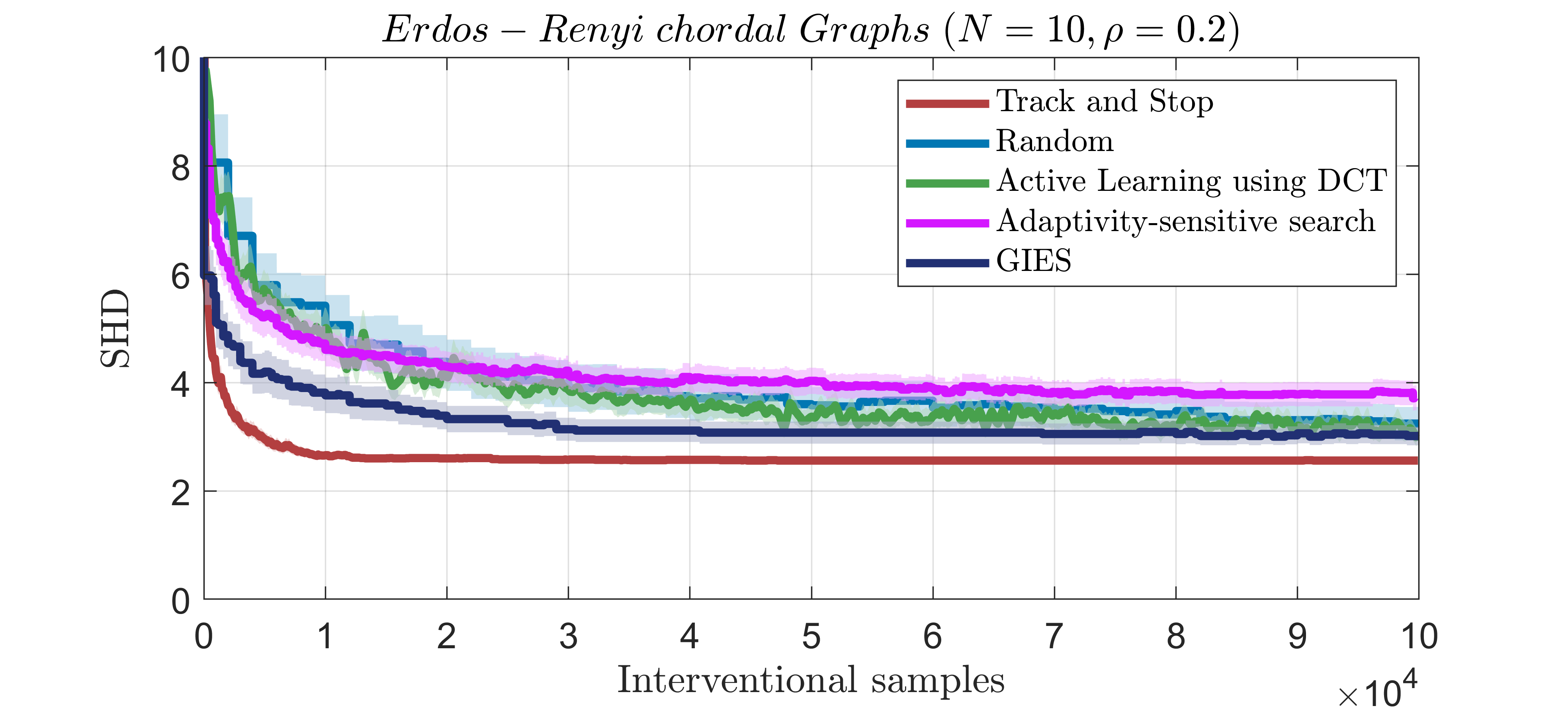}
\caption{SHD versus interventional samples for the discovery algorithms for Erdös-Rényi random chordal graphs starting with the incorrect CPDAG.}
\vspace{-1.5em}
\label{add_exp}
\end{figure}

Since we start from a incorrect CPDAG other than the true CPDAG due to limited observational data, the algorithms don't settle to zero SHD. Note that our track-and-stop causal discovery algorithm converges faster in terms of interventional samples compared to other baseline algorithms. In short, the track-and-stop causal discovery algorithm can still be employed to draw some reasonable conclusions about the data generation process with limited observational data.

\end{document}